\documentclass[11pt]{article}

\usepackage[utf8]{inputenc} % allow utf-8 input
\usepackage[T1]{fontenc}    % use 8-bit T1 fonts
\usepackage{hyperref}       % hyperlinks
\usepackage{url}            % simple URL typesetting
\usepackage{booktabs}       % professional-quality tables
\usepackage{nicefrac}       % compact symbols for 1/2, etc.
\usepackage{microtype}      % microtypography

\advance \oddsidemargin by -0.8in
\newcommand{\myparskip}{3pt}
\parskip \myparskip

 \usepackage{fullpage}
\usepackage[numbers]{natbib}
\usepackage{algorithm}
\usepackage[noend]{algorithmic}
\usepackage{amsmath,amsthm,amsfonts,amssymb}
\usepackage{xcolor}
\usepackage{color}
\usepackage{mathrsfs}
\usepackage{enumitem}
\usepackage{bm}
\usepackage{multirow}
\usepackage{booktabs}
\usepackage{makecell}
\usepackage{graphicx}
\usepackage{subfigure}
\usepackage{caption}
\usepackage{thmtools}
\usepackage{mathtools}
\usepackage{thm-restate}
\usepackage{tikz}
\usepackage{array}
\newcolumntype{P}[1]{>{\centering\arraybackslash}p{#1}}

\newcommand{\vect}[1]{\ensuremath{\mathbf{#1}}}
\newcommand{\mat}[1]{\ensuremath{\mathbf{#1}}}

\newcommand{\argmax}{\mathop{\rm argmax}}

\newcommand{\trans}{^{\top}}
\newcommand{\sdinv}{^{\dagger}}
\newcommand{\poly}{\mathrm{poly}}

\newcommand{\norm}[1]{\|{#1} \|}
\newcommand{\fnorm}[1]{\|{#1} \|_{\text{F}}}

\newcommand{\E}{\mathbb{E}}
\renewcommand{\Pr}{\mathbb{P}}

\newcommand{\pr}{^{\prime}}
\newcommand{\diag}{{\rm diag}}

\newcommand{\one}{{\mathbf1}}

\newcommand{\Ohat}{\hat{\O}}
\newcommand{\Phat}{\hat{P}}
\newcommand{\That}{\hat{\T}}

\newcommand{\thetahat}{\hat{\theta}}

\newcommand{\Cpoly}{C_{\rm poly}}

\newcommand{\at}{\tilde{a}}

\newcommand{\eps}{\varepsilon}

\newcommand{\bigO}{\mathcal{O}}
\newcommand{\tlO}{\mathcal{\tilde{O}}}

\newcommand{\N}{\mathbb{N}}
\newcommand{\R}{\mathbb{R}}

\newcommand{\B}{\mat{B}}

\newcommand{\I}{\mat{I}}

\newcommand{\V}{\mat{V}}
\newcommand{\W}{\mat{W}}
\newcommand{\X}{\mat{X}}
\newcommand{\Y}{\mat{Y}}

\newcommand{\e}{\vect{e}}
\renewcommand{\u}{\vect{u}}
\renewcommand{\v}{\vect{v}}

\newcommand{\x}{\vect{x}}
\newcommand{\y}{\vect{y}}
\newcommand{\z}{\vect{z}}

\newcommand{\nn}{\nonumber}
\usepackage{times}

\newtheorem{theorem}{Theorem}
\newtheorem{lemma}[theorem]{Lemma}
\newtheorem{fact}[theorem]{Fact}
\newtheorem{corollary}[theorem]{Corollary}

\theoremstyle{definition}

\newtheorem{assumption}{Assumption}

\newcommand{\cS}{\mathscr{S}}
\newcommand{\cA}{\mathscr{A}}
\newcommand{\cO}{\mathscr{O}}
\newcommand{\cT}{\mathscr{T}}
\renewcommand{\b}{\vect{b}}

\newcommand{\T}{\mathbb{T}}
\newcommand{\bp}{\vect{n}}
\newcommand{\bP}{\mat{N}}
\newcommand{\bQ}{\mat{M}}
\renewcommand{\O}{\mathbb{O}}
\newcommand{\pomdp}{\text{POMDP}}

\hypersetup{
    colorlinks,
    linkcolor={blue!50!black},
    citecolor={blue!50!black},
}
\colorlet{linkequation}{blue}

\begin{document}

% \title{\textbf{Statistical Efficiency of POMDP}}
% \title{\textbf{Provably Sample-Efficient Algorithms for Learning Partially Observable Markov Decision Process}}
\title{\textbf{Sample-Efficient Reinforcement Learning\\
 of Undercomplete POMDPs}}

\author{Chi Jin \\ Princeton University\\ {\tt chij@princeton.edu}
\and Sham M. Kakade \\ University of Washington \\
Microsoft Research, NYC \\
{\tt sham@cs.washington.edu }
\and Akshay Krishnamurthy \\Microsoft Research, NYC \\  {\tt  akshaykr@microsoft.com} 
\and Qinghua Liu \\ Princeton University \\{\tt qinghual@princeton.edu} 
}

\maketitle

\newcommand{\sk}[1]{\noindent{\textcolor{magenta}{\{{\bf SK:} \em #1\}}}}
\newcommand{\chijin}[1]{\noindent{\textcolor{magenta}{\{{\bf CJ:} \em #1\}}}}
\newcommand{\qinghua}[1]{\noindent{\textcolor{red}{\{{\bf QL:} \em #1\}}}}
\newcommand{\akshay}[1]{\noindent{\textcolor{red}{\{{\bf AK:} \em #1\}}}}

\begin{abstract}
Partial observability is a common challenge in many reinforcement learning applications, which requires an agent to maintain memory, infer latent states, and integrate this past information into exploration. 
This challenge leads to a number of computational and statistical hardness
results for learning general Partially Observable Markov Decision Processes
(POMDPs).
This work shows that these hardness barriers do not
preclude efficient reinforcement learning for rich and interesting 
subclasses of POMDPs. In
particular, we present a sample-efficient algorithm, \emph{OOM-UCB}, for
episodic finite \emph{undercomplete} POMDPs, where the number of observations is larger than the number of latent states
and where exploration is essential for learning, thus distinguishing
our results from prior works. \emph{OOM-UCB} achieves an optimal sample complexity of
$\tilde{\mathcal{O}}(1/\varepsilon^2)$ for finding an $\varepsilon$-optimal policy, along with being polynomial in all other relevant quantities. As an interesting
special case, we also provide a computationally and statistically
efficient algorithm for POMDPs with deterministic state transitions.

\end{abstract}

\section{Introduction}

In many sequential decision making settings, the agent lacks complete
information about the underlying state of the system, a phenomenon
known as \emph{partial observability}.  Partial observability
significantly complicates the tasks of reinforcement learning and
planning, because the non-Markovian nature of the observations forces
the agent to maintain memory and reason about beliefs of the system
state, all while exploring to collect information about the
environment. For example, a robot may not be able to perceive all
objects in the environment due to occlusions, and it must reason about
how these objects may move to avoid
collisions~\citep{cassandra1996acting}. Similar reasoning problems
arise in imperfect information games~\citep{brown2018superhuman},
medical diagnosis~\citep{hauskrecht2000planning}, and
elsewhere~\citep{rafferty2011faster}.  
%Such
%reasoning is quite challenging, and indeed, 
Furthermore, from a theoretical perspective, well-known
complexity-theoretic results show that learning and planning in
partially observable environments is statistically and computationally
intractable in
general~\citep{papadimitriou1987complexity,mundhenk2000complexity,vlassis2012computational,mossel2005learning}.

The standard formulation for reinforcement learning with partial
observability is the \emph{Partially Observable Markov Decision
  Process} (POMDP), in which an agent operating on noisy observations
makes decisions that influence the evolution of a latent state. The
complexity barriers apply for this model, but they are of a worst case
nature, and they do not preclude
%% While discouraging, the complexity barriers are of a worst-case
%% nature, and they do not preclude 
efficient algorithms for interesting sub-classes of POMDPs. Thus we ask:
\begin{center}
\emph{Can we develop efficient algorithms for reinforcement learning in large classes of POMDPs?}
\end{center}
This question has been studied in recent
works~\citep{azizzadenesheli2016reinforcement,guo2016pac}, which
incorporate a decision making component into a long line of work on
``spectral methods'' for estimation in latent variable
models~\citep{hsu2012spectral,song2010hilbert,anandkumar2012method,anandkumar2014tensor},
including the Hidden Markov Model. Briefly, these estimation results
are based on the method of moments, showing that under certain
assumptions the model parameters can be computed by a decomposition of
a low-degree moment tensor. The works of Azizzadenesheli et
al.~\citep{azizzadenesheli2016reinforcement} and Guo et
al.~\citep{guo2016pac} use tensor decompositions in the POMDP setting
and obtain sample efficiency guarantees. Neither result
considers a setting where strategic exploration is essential for
information acquisition, and they do not address one of the central
challenges in more general reinforcement learning problems.

%% identify certain use the method-of-moments
%%  This
%% viewpoint has inspired much work on estimation in latent variable
%% models without the decision making component~\citep{}, and even some
%% recent results for POMDPs~\citep{}. 
%% However, none of these works consider settings where strategic
%% exploration is essential for information acquisition, and so they do
%% not address one of the central challenges in reinforcement learning.

\paragraph{Our contributions.}
In this work, we provide new sample-efficient algorithms for
reinforcement learning in finite POMDPs in the \emph{undercomplete}
regime, where the number of observations is larger than the number of
latent states. This assumption is quite standard in the literature on
estimation in latent variable models~\citep{anandkumar2014tensor}.  Our main algorithm
\emph{OOM-UCB} uses the principle of optimism for exploration and uses
the information gathered to estimate the \emph{Observable
  Operators} induced by the environment. Our main result proves that
\emph{OOM-UCB} finds a near optimal policy for the POMDP using a
number of samples that scales polynomially with all relevant
parameters and additionally with the minimum singular value of the
emission matrix. Notably, \emph{OOM-UCB} finds an
$\eps$-optimal policy at the optimal rate of
$\tilde{\mathcal{O}}(1/\eps^2)$.

\iffalse
While \emph{OOM-UCB} is statistically efficient, it is not
computationally efficient, as this would violate computational
barriers for POMDPs. We address this with a computationally and
statistically efficient algorithm for POMDPs in which the latent
dynamics are deterministic. While deterministic dynamics avoids
computational barriers, it does not mitigate the need for
exploration. We prove that this algorithm has sample complexity
scaling with all the relevant parameters as well as the minimum
$\ell_2$ distance between emission distributions. This latter
parameter, which replaces the minimum singular value in the guarantee
for \emph{OOM-UCB}, is more favorable, and so this algorithm is
preferrable for deterministic POMDPs.
\fi

While \emph{OOM-UCB} is statistically efficient for this subclass of
POMDPs, we should not expect it to be computationally efficient in general, as this would violate computational
barriers for POMDPs.  However, in our second contribution, we consider a further 
restricted subclass of POMDPs in which the latent
dynamics are deterministic and where we provide \emph{both}
a computationally and statistically efficient algorithm. Notably,
deterministic dynamics are still an interesting subclass due to that,
while it avoids
computational barriers, it still does not mitigate the need for
strategic exploration. We prove that our second algorithm has sample complexity
scaling with all the relevant parameters as well as the minimum
$\ell_2$ distance between emission distributions. This latter
quantity replaces the minimum singular value in the guarantee
for \emph{OOM-UCB} and is a more favorable dependency.
%, and so this algorithm is preferrable for deterministic POMDPs.

We provide further motivation for our assumptions with two lower
bounds: the first shows that the overcomplete setting is statistically
intractable without additional assumptions, while the second
necessitates the dependence on the minimum singular value of the
emission matrix. In particular, under our assumptions, the agent must
engage in strategic exploration for sample-efficiency. As such, the
main conceptual advance in our line of inquiry over prior works is
that our algorithms address exploration and partial observability in a
provably efficient manner.

\iffalse
Our assumptions are significantly weaker than prior work for
reinforcement learning in POMDPs (discussed in detail in the related
work section), and we provide further justification with two lower
bounds.  The first shows that the overcomplete setting is
statistically intractable without additional assumptions, while the
second necessitates the dependence on the minimum singular value of
the emission matrix. In particular, under our assumptions, the agent
must engage in strategic exploration for sample-efficiency. As such,
the main conceptual advance in our work is that our algorithms address
exploration and partial observability in a provably efficient manner.
\fi

%!TEX root = main.tex

\subsection{Related work}

A number of computational barriers for POMDPs are known. If the
parameters are known, it is PSPACE-complete to compute the optimal
policy, and, furthermore, it is NP-hard to compute the optimal
memoryless
policy~\cite{papadimitriou1987complexity,vlassis2012computational}.
With regards to learning, Mossel and Roch~\citep{mossel2005learning}
provided an average case computationally complexity result, showing
that parameter estimation for a subclass of Hidden Markov Models (HMMs)
is at least as hard as learning parity with noise. This directly
implies the same hardness result for parameter estimation in POMDP
models, due to that an HMM is just a POMDP with a fixed action
sequence.  On the other hand, for reinforcement learning in POMDPs (in
particular, finding a near optimal policy), one may not need to
estimate the model, so this lower bound need not directly imply that
the RL problem is computational intractable. In this work, we do provide a
lower bound showing that reinforcement learning in POMDPs is both statistically
and computationally intractable (Propositions~\ref{prop:hardness1} and \ref{prop:hardness2}).

On the positive side, there is a long history of work on learning
POMDPs. 
%~\cite{ng2000pegasus} proposed a policy search algorithm
%assuming access to a deterministic simulator, an assumption
%that is unlikely to hold in applications.
\cite{even2005reinforcement} studied POMDPs without resets, where the
proposed algorithm has sample complexity scaling exponentially with a
certain horizon time, which is not possible to relax without further restrictions.~\cite{ross2008bayes,poupart2008model} proposed
to learn POMDPs using Bayesian methods; PAC or regret bounds
are not known for these approaches.
\cite{azizzadenesheli2018policy} studied policy gradient methods for learning POMDPs while they considered only Markovian policies and did not address exploration.

%% Although Bayesian methods have
%% PAC results for MDPs (\cite{kolter2009near}), no PAC bound nor regret
%% bound is known for POMDPs.

Closest to our work are POMDP algorithms based on spectral
methods~\citep{guo2016pac,azizzadenesheli2016reinforcement}, which
were originally developed for learning latent variable
models~\citep{hsu2012spectral,anandkumar2012method,anandkumar2014tensor,song2010hilbert,sharan2017learning}.
These works give PAC and regret bounds (respectively) for tractable
subclasses of POMDPs, but, in contrast with our work,
%% The main
%% difference between \cite{guo2016pac,azizzadenesheli2016reinforcement}
%% and the present paper is that t
they make additional assumptions to mitigate the exploration
challenge.  In~\citep{guo2016pac}, it is assumed that all latent
states can be reached with nontrivial probability with a constant
number of random actions. This allows for estimating the \emph{entire}
model without sophisticated exploration.
\cite{azizzadenesheli2016reinforcement} consider a special class of
memoryless policies in a setting where all of these policies visit
every state and take every action with non-trivial probability. As
with~\citep{guo2016pac}, this restriction guarantees that the entire
model can be estimated regardless of the policy executed, so
sophisticated exploration is not required.
We also mention
that~\citep{guo2016pac,azizzadenesheli2016reinforcement} assume that
both the transition and observation matrices are full rank, which is
stronger than our assumptions. We do not make any assumptions on the
transition matrix.

Finally, the idea of representing the probability of a sequence as
products of operators dates back to multiplicity
automata~\cite{schutzenberger1961definition, carlyle1971realizations}
and reappeared in the Observable Operator Model
(OOMs)~\cite{jaeger2000observable} and Predictive State
Representations (PSRs)~\citep{littman2002predictive}. While spectral
methods have been applied to PSRs~\citep{boots2011closing}, we are not
aware of results with provable guarantees using this approach. It is also worth mentioning that
any POMDP can be modeled as an Input-Output OOM~\cite{jaeger1998discrete}.
% \cite{singh2012predictive}.

%!TEX root = main.tex

\section{Preliminaries}

In this section, we define the partially observable Markov decision process, the observable operator model~\cite{jaeger2000observable}, and discuss their relationship.

\paragraph{Notation.} For any natural number $n \in \N$, we use $[n]$ to denote the set $\{1, 2, \ldots, n\}$. We use bold upper-case letters $\B$ to denote matrices and bold lower-case letters $\b$ to denote vectors. $\B_{ij}$ means the $(i, j)^{\text{th}}$ entry of matrix $\B$ and $(\B)_i$ represents its $i^{\rm th}$ column. For vectors we use $\norm{\cdot}_p$ to denote the $\ell_p$-norm, and for matrices we use $\norm{\cdot}$, ${\norm{\cdot}}_1$ and $\fnorm{\cdot}$ to denote the spectral norm, entrywise $\ell_1$-norm and Frobenius norm respectively.
We denote by $\| \B\|_{p\rightarrow q}=\max_{\|\v\|_p \le 1}\|\B\v\|_q$ the $p$-to-$q$ norm of $\B$.
% \chijin{define other used matrix norm here.} 
For any matrix $\B\in\R^{m \times n}$, we use $\sigma_{\min}(\B)$ to denote its smallest singular value, and $\B^\dagger \in \R^{n \times m}$ to denote its Moore-Penrose inverse. For vector $\v \in \R^n$, we denote $\diag(\v) \in \R^{n\times n}$ as a diagonal matrix where $[\diag(\v)]_{ii} = \v_i$ for all $i \in [n]$. Finally, we use standard big-O and big-Omega notation $\bigO(\cdot), \Omega(\cdot)$ to hide only absolute constants which do not depend on any problem parameters, and notation $\tlO(\cdot), \tilde{\Omega}(\cdot)$ to hide only absolute constants and logarithmic factors. 

% \chijin{Introduce notation for matrices, diag, etc}

\subsection{Partially observable Markov decision processes}

We consider an episodic tabular Partially Observable Markov Decision Process (POMDP), which can by specified as $\pomdp(H, \cS, \cA, \cO, \T, \O, r, \mu_1)$. Here $H$ is the number of steps in each episode, $\cS$ is the set of states with $|\cS| = S$, $\cA$ is the set of actions with $|\cA| = A$, $\cO$ is the set of observations with $|\cO| = O$, $\T = \{\T_h\}_{h=1}^H$ specify the transition dynamics such that $\T_h(\cdot|s, a)$ is the distribution over states if action $a$ is taken from state $s$ at step $h \in [H]$, $\O = \{\O_h\}_{h=1}^H$ are emissions such that $\O_h(\cdot|s)$ is the distribution over observations for state $s$ at step $h \in [H]$, $r =  \{r_h: \cO \rightarrow [0, 1]\}_{h=1}^H$ are the known deterministic reward functions\footnote{Since rewards are observable in most applications, it is natural to assume the reward is a known function of the observation. While we study deterministic reward functions for notational simplicity, our results generalize to randomized reward functions. Also, we assume the reward is in $[0, 1]$ without loss of generality.}, and $\mu_1(\cdot)$ is the initial distribution over states. Note that we consider nonstationary dynamics, observations, and rewards. 

% \chijin{Discuss different assumptions on reward function and equivalence to those assumed as to be an function of action by augmented POMDP, maybe in appendix.}

In a POMDP, states are hidden and unobserved to the learning agent. Instead, the agent is only able to see the observations and its own actions. At the beginning of each episode, an initial hidden state $s_1$ is sampled from initial distribution $\mu_1$. At each step $h \in [H]$, the agent first observes $o_h \in \cO$ which is generated from the hidden state $s_h \in \cS$ according to $\O_h(\cdot|s_h)$, and receives the reward $r_h(o_h)$, which can be computed from the observation $o_h$. Then, the agent picks an action $a_h \in \cA$, which causes the environment to transition to hidden state $s_{h+1}$, that is drawn from the distribution $\T_h(\cdot|s_h, a_h)$. The episode ends when $o_{H}$ is observed. 
% We remark that since we assume the rewards are observable (thus functions of observations), the rewards are already implicitly collected from observations.
% \akshay{Do we want to mention rewards here?}

A policy $\pi$ is a collection of $H$ functions $\big\{ \pi_h: \cT_h \rightarrow \cA \big\}_{h\in [H]}$, where $\cT_h = (\cO \times \cA)^{h-1} \times \cO$ is the set of all possible histories of length $h$.
We use $V^{\pi} \in \mathbb{R}$ to denote the value of policy $\pi$, so that $V^\pi$ gives the expected cumulative reward received under policy $\pi$:
\begin{equation*}
\textstyle V^\pi := \E_\pi \left[\sum_{h = 1}^H r_{h}(o_{h})\right].
\end{equation*}
Since the state, action, observation spaces, and the horizon, are all finite, there always exists an optimal policy $\pi^\star$ which gives the optimal value $V^\star = \sup_{\pi} V^\pi$. We remark that, in general, the optimal policy of a POMDP will select actions based the entire history, rather than just the recent observations and actions. This is one of the major differences between POMDPs and standard Markov Decision Processes (MDPs), where the optimal policies are functions of the most recently observed state. 
%% of the latter model pick actions only based on the last observed state. 
This difference makes POMDPs significantly more challenging to solve.

\paragraph{The POMDP learning objective.} Our objective in this paper
is to learn an \textbf{$\varepsilon$-optimal policy} $\hat{\pi}$ in the
sense that $ V^{\hat{\pi}} \geq V^\star  -\varepsilon$, using a polynomial
number of samples.  

% \begin{itemize}
% \item $\cS$: the set of states with $|\cS| = S$
% \item $\cA$: the set of actions with $|\cA| = A$.
% \item $\cO$: the set of observations with $|\cO| = O$.
% \item $\T$: transition probability matrix where $\T(\cdot|s, a)$ is the distribution over states if action $a$ is taken for state $s$.
% \item $\O$: observation probability matrix where $\O(\cdot|s)$ is the distribution over observations for state $s$.
% \item $r: \cO \rightarrow [0, 1]$ is the reward function.
% \item $H$: is the finite horizon. 
% \item $\mu_1$: the distribution of initial states
% \end{itemize}

\subsection{The observable operator model} \label{sec:OOM}

% We introduce here a (generalized version) of Input-Output Observable Operator Models (IO-OOM) \cite{jaeger2000observable}.
% % , which serves as a generalization of the standard POMDP models. 
% % The new perspective from IO-OOM turns out to be crucial for our algorithm design. 
% % We also refer to IO-OOM as OOM for short in this paper. 
% Formally, the IO-OOM model considered in this paper consists of a set of observable operators $\{\B_h(a,o) \in \R^{O\times O}\}_{h\in[H], a\in \cA, o \in \cO}$ and a vector $\b_0 \in \R^O$, such that:\footnote{We remark the model introduced in this paper has minor difference from the standard IO-OOM \cite{jaeger2000observable} which further requires the column sum of the observable operator equals to $1$.}
% \begin{equation*}
% \Pr(o_H, \ldots, o_1 | a_{H-1}, \ldots, a_1) = \e_{o_H}\trans \cdot \B_{H-1}(a_{H-1}, o_{H-1}) \cdots \B_1(a_1, o_1) \cdot \b_0.
% \end{equation*}

% The observable operators and vector $\b_0$ are sufficiently

We have described the POMDP model via the transition and observation
distributions $\T,\O$ and the initial distribution $\mu_1$. While this
parametrization is natural for describing the dynamics of the system,
POMDPs can also be fully specified via a different set of parameters: 
%% using parameters in transition and observation matrices $\T, \O$ and the initial distribution $\mu_1$. Alternatively, POMDP can also be fully specified by a different set of parameters: 
a set of operators $\{\B_h(a,o) \in \R^{O\times O}\}_{h\in[H-1], a\in \cA, o \in \cO}$, and a vector $\b_0 \in \R^O$.

In the undercomplete setting where $S\le O$ and where observation probability matrices $\{\O_h \in \R^{O\times S}\}_{h\in [H]}$ are all full column-rank, the operators $\{\B_h(a, o)\}_{h, a, o}$ and vector $\b_0$ can be expressed in terms of $(\T, \O, \mu_1)$ as follows:
\begin{equation}\label{eq:OM_relation}
\B_h(a, o) =  \O_{h+1} \T_h(a) \diag(\O_h(o|\cdot)) \O_h^{\dagger}, \qquad \b_0 = \O_1  \mu_1.
\end{equation}
where we use the matrix and vector notation for $\O_h \in \R^{O\times S}$ and $\mu_1 \in \R^S$ here, 
such that $[\O_h]_{o, s} = \O_h(o|s)$ and $[\mu_1]_s = \mu_1(s)$. $\T_h(a) \in \R^{S\times S}$ denotes the transition matrix given action $a \in \cA$ where $[\T_h(a)]_{s',s} = \T_h(s'|s, a)$, and $\O_h(o|\cdot) \in \R^S$ denotes the $o$-th row in matrix $\O_h$ with $[\O_h(o|\cdot)]_s = \O_h(o|s)$. 
Note that the matices defined in \eqref{eq:OM_relation} have rank at most $S$. 
Using these matrices $\B_h$, it can be shown that (Appendix \ref{app:pomdp-oom}), for any sequence of $(o_H, \ldots, a_1, o_1) \in \cO\times (\cA \times \cO)^{H-1}$, we have:
% \begin{equation*}
% \Pr(o_H, \ldots, o_1 | a_H, \ldots, a_1) = \b_{\infty}\trans \cdot \B_H(a_H, o_H) \cdots \B_1(a_1, o_1) \cdot \b_0.
% \end{equation*}
\begin{equation} \label{eq:OOM_pb}
\Pr(o_H, \ldots, o_1 | a_{H-1}, \ldots, a_1) = \e_{o_H}\trans \cdot \B_{H-1}(a_{H-1}, o_{H-1}) \cdots \B_1(a_1, o_1) \cdot \b_0.
\end{equation}
Describing these conditional probabilities for every sequence is sufficient to fully specify the entire dynamical system. Therefore, as an alternative to directly learning $\T, \O$ and $\mu_1$, it is also sufficient to learn operators $\{\B_h(a,o)\}_{h, a, o}$ and vector $\b_0$ in order to learn the optimal policy. The latter approach enjoys the advantage that \eqref{eq:OOM_pb} does not explicitly involve latent variables. It refers only to observable quantities---actions and observations.

% One major advantage of using operator model is that the does not involve the latent variables 
% \akshay{Explain why this is helpful? e.g., ``The advantage of this reparametrization is that~\eqref{eq:OOM_pb} refers only to observable quantities.''}

We remark that the operator model introduced in this section (which is parameterized by $\{\B_h(a,o)\}_{h, a, o}$ and $\b_0$) bears significant similarity to Jaeger's Input-Output Observable Operator Model (IO-OOM) \cite{jaeger2000observable}, except a few minor technical differences.\footnote{Jaeger's IO-OOM further requires the column-sums of operators to be 1.}  With some abuse of terminology, we also refer to our model as Observable Operator Model (OOM) in this paper. 
It is worth noting that Jaeger's IO-OOMs are strictly more general than POMDPs \cite{jaeger2000observable} and also includes overcomplete POMDPs via a relation different from~\eqref{eq:OM_relation}. 
%% that one can cast an overcomplete POMDP as an IO-OOM model using a relation different from \eqref{eq:OM_relation}. 
Since our focus is on undercomplete POMDPs, we refer the reader to~\cite{jaeger2000observable} for more details.

% \footnote{One can also cast an overcomplete POMDP as an OOM model, but with a different relation than \eqref{eq:OM_relation}.}

% \akshay{A question that is not quite coming across now: the $\B_h$ representation applies even in the overcomplete setting, it is just less clear how to express $\B_h$ in terms of $\T,\O,\mu_1$. Is that right?}
% \chijin{Remark the difference to the standard Observable Operator Model (OOMs), maybe say this is the general OOMs.}

% \akshay{We are assuming $O \geq S$ and that $\O$ is rank $S$, right?}

% \qinghua{Yes.}

% \subsection{Operator Model}

%!TEX root = main.tex

\section{Main Results}

We first state our main assumptions, which we motivate with
corresponding hardness results in their absence. We then present our
main algorithm, \emph{OOM-UCB}, along with its sample efficiency guarantee.   

% \paragraph{Main assumptions.}  
\subsection{Assumptions}
In this paper, we make the following assumptions.

\begin{assumption}\label{assump:POMDP}
We assume the POMDP is undercomplete, i.e. $S \le O$. We also assume the minimum singular value of the observation probability matrices $\sigma_{\min}(\O_h) \ge \alpha>0$ for all $h \in [H]$.
\end{assumption}

Both assumptions
%%   and minimum singular value assumption 
are standard in the literature on learning Hidden Markov Models (HMMs)---an uncontrolled version of POMDP~\cite[see e.g.,][]{anandkumar2012method}. The second assumption that $\sigma_{\min}(\O_h)$ is lower-bounded is a robust version of the assumption that $\O_h \in \R^{O\times S}$ is full column-rank, which is equivalent to $\sigma_{\min}(\O_h)>0$. 
% \akshay{minimum singular value also appears in HMM literature.} 
Together, these assumption ensure that the observations will contain a reasonable amount of information about the latent states.
% Combining two assumptions ensures that a reasonable amount of information about the latent states will be conveyed into the observations.

We do not assume that the initial distribution $\mu_1$ has full support, nor do we assume the transition probability matrices $\T_h$ are full rank.
In fact, Assumption \ref{assump:POMDP} is \emph{not} sufficient for identification of the system, i.e. recovering parameters $\T, \O, \mu_1$ in total-variance distance. 
Exploration is crucial to find a near-optimal policy in our setting.

% Assumption \ref{assump:POMDP} enables the use of OOM introduced in section \ref{sec:OOM}

% We make two key assumptions: 1. undercomplete, 2. $\O$ is full row-rank

We motivate both assumptions above by showing that, with absence of
either one, learning a POMDP is statistically intractable. That is, it
would require an exponential number of samples for any algorithm to
learn a near-optimal policy with constant probability.  
% Therefore,ere will not be sample-efficient algorithms for these settings.

\begin{restatable}{proposition}{LowerBoundOvercomplete}
\label{prop:hardness1}
For any algorithm $\mathfrak{A}$, there exists an overcomplete POMDP ($S>O$) with $S$ and $O$ being small constants, which satisfies $\sigma_{\min}(\O_h) =1$ for all $h \in [H]$, such that algorithm $\mathfrak{A}$ requires at least $\Omega(A^{H-1})$ samples to ensure learning a $(1/4)$-optimal policy with probability at least $1/2$. 
\end{restatable}

\begin{restatable}{proposition}{LowerBoundUndercomplet}
\label{prop:hardness2}
For any algorithm $\mathfrak{A}$, there exists an undercomplete POMDP ($S\le O$) with $S$ and $O$ being small constants, such that algorithm $\mathfrak{A}$ requires at least $\Omega(A^{H-1})$ samples to ensure learning a $(1/4)$-optimal policy with probability at least $1/2$. 
% For any algorithm, there exists an undercomplete POMDP where learning $(1/2)$-optimal policy with probability greater than $1/2$ requires at least $\Omega(A^{H/2})$ samples.
\end{restatable}

Proposition \ref{prop:hardness1} and \ref{prop:hardness2} are both proved by constructing hard instances, which are modifications of classical combinatorial locks for MDPs~\cite{krishnamurthy2016pac}. We refer readers to Appendix \ref{app:hardness} for more details.
% Therefore, these settings are statistically hard.
% \chijin{Talk more about lower bound construction?}
% \chijin{Comment on that we do not require identifiability? transition being full rank?}

% \akshay{Do we need to say that $S$ and $O$ are constants in both propositions?}

\begin{algorithm}[t]
\caption{Observable Operator Model with Upper Confidence Bound (OOM-UCB)}
\label{alg:UCB}
\begin{algorithmic}[1]
\STATE {\bfseries Initialize:} set all entries in a vector of counts $\bp \in \N^O$, and in matrices of counts $\bP_h(a, \tilde{a}) \in \N^{O\times O}$, $\bQ_h(o, a, \tilde{a}) \in \N^{O\times O}$ to be zero for all $(o, a,\tilde{a})\in \cO \times \cA^2$
\STATE set confidence set $\Theta_1 \leftarrow \cap_{h\in[H]}\{\hat{\theta} ~|~ \sigma_{\min} (\hat{\O}_h) \ge \alpha\}$.
\FOR{$k=1,2,\ldots, K$}
\STATE compute the optimistic policy $\pi_k \leftarrow \argmax_{\pi} \max_{\hat{\theta} \in \Theta_k} V^\pi(\hat{\theta})$. \label{line:greedy}
\STATE observe $o_1$, and set $\bp \leftarrow \bp + \e_{o_1}$ \label{line:conf_start}
\STATE $\mathfrak{b} \leftarrow (\cap_{h\in[H]}\{\hat{\theta} ~|~ \sigma_{\min} (\hat{\O}_h) \ge \alpha\}) \cap
\{\hat{\theta} ~|~ \norm{k \cdot \b_0(\hat{\theta}) - \bp}_2 \le  \beta_k\} $. \label{line:conf_b}
\FOR{$(h, a,\tilde{a})\in[H-1] \times \cA^2$}
 \STATE execute policy $\pi_k$ from step $1$ to step $h-2$.
 \STATE take action $\tilde{a}$ at step $h-1$, and action $a$ at step $h$ respectively.
 \STATE observe $(o_{h-1},o_h,o_{h+1})$, and set $\bP_h(a, \tilde{a}) \leftarrow \bP_h(a, \tilde{a}) + \e_{o_{h}}\e_{o_{h-1}}\trans$.
 \STATE set $\bQ_h(o_h, a, \tilde{a}) \leftarrow \bQ_h(o_h, a, \tilde{a}) + \e_{o_{h+1}}\e_{o_{h-1}}\trans$.
 \STATE $\mathfrak{B}_h(a, \tilde{a}) \leftarrow \cap_{o \in \cO} \{\hat{\theta} ~|~ 
 \norm{\B_h(a, o; \hat{\theta}) \bP_{h} (a, \tilde{a}) -\bQ_h (o, a, \tilde{a})}_F \le \gamma_k \}.$  \label{line:conf_B}
%  \STATE $X_{k+1} (a\pr,a) = X_k(a\pr,a) + e_{o_{h}}e_{o_{h-1}}\trans$
%  \FOR{$ x \in \Omega$}
% \STATE $Y_{k+1} (a\pr,a,x) = Y_k(a\pr,a,x) + e_{o_{h+1}}e_{o_{h-1}}\trans \one(e_{o_{h}}=x)$
% \ENDFOR
\ENDFOR
\STATE construct the confidence set $\Theta_{k+1} \leftarrow [\cap_{(h, a,\tilde{a})\in[H-1] \times \cA^2}\mathfrak{B}_h(a, \tilde{a}) ] \cap \mathfrak{b}$. \label{line:conf_end}
\ENDFOR
\STATE {\bfseries Output:} $\pi_k$ where $k$ is sampled uniformly from $[K]$.
\end{algorithmic}
\end{algorithm} 

\subsection{Algorithm}
We are now ready to describe our algorithm. Assumption \ref{assump:POMDP} enables the representation of the POMDP using OOM with relation specified as in Equation \eqref{eq:OM_relation}. 
% In fact, we design our algorithm heavily exploiting this OOM representation. 
Our algorithm, Observable Operator Model with Upper Confidence Bound
(OOM-UCB, algorithm \ref{alg:UCB}), is an optimistic algorithm which
heavily exploits the OOM representation to obtain valid uncertainty
estimates of the parameters of the underlying model.

To condense notation in Algorithm~\ref{alg:UCB}, we denote the
parameters of a POMDP as $\theta = (\T, \O, \mu_1)$. We denote
$V^\pi(\theta)$ as the value of policy $\pi$ if the underlying POMDP
has parameter $\theta$. We also write the parameters of the OOM
$(\b_0(\theta), \B_h(a, o; \theta))$ as a function of parameter
$\theta$, where the dependency is specified as
in~\eqref{eq:OM_relation}. We adopt the convention that at the
$0$-th step, the observation $o_0$ and state $s_0$ are always set to
be some fixed dummy observation and state, and, starting from $s_0$, the
environment transitions to $s_1$ with distribution $\mu_1$ regardless
of what action $a_0$ is taken. 

At a high level, Algorithm~\ref{alg:UCB} is an iterative algorithm
that, in each iteration, (a) computes an optimistic policy and model
by maximizing the value (Line \ref{line:greedy}) subject to a given
confidence set constraint, (b) collects data using the optimistic
policy, and (c) incorporates the data into an updated confidence set
%% and constructing a
%% confidence set 
for the OOM parameters (Line
\ref{line:conf_start}-\ref{line:conf_end}).  The first two parts are
straightforward, so we focus the discussion on computing the
confidence set. We remark that in general the optimization in Line
\ref{line:greedy} may not be solved in polynomial time (see discussions of the computational complexity after Theorem \ref{thm:main}).

 % \akshay{Should we remark that Line~\ref{line:greedy} is where the computational problem is?}

First, since $\b_0$ in \eqref{eq:OM_relation} is simply the
probability over observations at the first step, our confidence set
for $\b_0$ in Line \ref{line:conf_b} is simply based on counting the
number of times each observation appears in the first step and
Hoeffding's concentration inequality.

Our construction of the confidence sets for the operators $\{\B_h(a,
o)\}_{h, a, o}$ is inspired by the method-of-moments estimator in HMM
literature~\citep{hsu2012spectral}. 
Consider two fixed actions $a, \tilde{a}$, and an arbitrary
distribution over $s_{h-1}$. Let $\mat{P}_h(a, \tilde{a}),
\mat{Q}_h(o, a, \tilde{a}) \in \R^{O\times O}$ be the probability
matrices such that  
\begin{align}
[\mat{P}_h(a, \tilde{a})]_{o', o''} =& \Pr(o_h = o', o_{h-1} = o'' |a_h=a, a_{h-1} = \tilde{a}), \nonumber\\
[\mat{Q}_h(o, a, \tilde{a})]_{o', o''} =& \Pr(o_{h+1} = o', o_h = o, o_{h-1} = o'' |a_h=a, a_{h-1} = \tilde{a}). \label{eq:PnQ}
\end{align}
It can be verified that $\B_h(a, o)\mat{P}_h(a, \tilde{a}) =
\mat{Q}_h(o, a, \tilde{a})$ (Fact \ref{fact:equation} in the appendix).
 Our confidence set construction (Line
\ref{line:conf_B} in Algorithm \ref{alg:UCB}) is based on this fact:
we replace the probability matrices $\mat{P}, \mat{Q}$ by empirical
estimates $\bP, \bQ$, and we use concentration inequalities to
determine the width of the confidence set. Finally, our overall
confidence set for the parameters $\theta$ is simply the intersection
of the confidence sets for all induced operators and $\b_0$,
additionally incorporating the constraint on $\sigma_{\min}(\O_h)$
from Assumption \ref{assump:POMDP}.

% , and the usage of concentration inequality. Finally, the 

% Our confidence set for 

% constructs confidence sets for parameters in the OOM, i.e. the operators and the vector $\b_0$.

% We denote the parameters of POMDP as $\theta = (\T, \O, \mu_1)$, we note the dependency of operator model parameters $(\b_0(\theta), \B_h(a, o; \theta))$ on $\theta$ is given as in \eqref{eq:OM_relation}.

% \chijin{Say the convention for $-1, 0, H, H+1$ steps.}

\subsection{Theoretical guarantees}

Our OOM-UCB algorithm enjoys the following sample complexity guarantee.

\begin{restatable}{theorem}{UCBTheorem}
\label{thm:main}
For any $\varepsilon \in (0, H]$, there exists $K_{\max} = \poly(H,S, A, O, \alpha^{-1})/\varepsilon^2$ and an absolute constant $c_1$, such that for any POMDP that satisfies Assumption \ref{assump:POMDP}, if we set hyperparameters $\beta_k = c_1\sqrt{k\log(KAOH)}$, $\gamma_k = \sqrt{S}\beta_k/\alpha$, and $K \ge K_{\max}$, then the output policy $\hat{\pi}$ of Algorithm \ref{alg:UCB} will be $\varepsilon$-optimal with probability at least $2/3$.
\end{restatable}

Theorem \ref{thm:main} claims that in polynomially many iterations of the outer loop, Algorithm \ref{alg:UCB} learns a near-optimal policy for any undercomplete POMDP that satisfies Assumption \ref{assump:POMDP}. 
Since our algorithm only uses $O(H^2A^2)$ samples per iteration of the outer loop, this implies that the sample complexity is also 
%% Theorem \ref{thm:main} also asserts that the sample complexity of Algorithm \ref{alg:UCB} is also 
$\poly(H, S, A, O, \alpha^{-1})/\varepsilon^2$. We remark that the $1/\varepsilon^2$ dependence is optimal, which follows from standard concentration arguments. To the best of our knowledge, this is the first sample efficiency result for learning a class of POMDPs where exploration is essential. 
\footnote{See Appendix \ref{app:proof-mainthm} for the explicit polynomial dependence of sample complexity;
here, the success probability  is a constant, but one can make it arbitrarily close to $1$ by a standard boosting trick (see Appendix \ref{app:justification3} ). }

While Theorem \ref{thm:main} does guarantee sample
efficiency, Algorithm~\ref{alg:UCB} is not computationally
efficient due to  that the computation of the optimistic policy (Line
\ref{line:greedy}) may not admit a polynomial time implementation,
which should be expected given the aforementioned computational
complexity results.  We now turn to a further restricted (and interesting) subclass of
POMDPs where we can address \emph{both} the computational and
statistical challenges. 

%Note that prior results for learning POMDPs are also based on computationally inefficient algorithms~\cite[see e.g.,][]{azizzadenesheli2016reinforcement,guo2016pac}.
%% Computational inefficiency is common for most existing results of learning POMDPs \cite[see e.g.][]{azizzadenesheli2016reinforcement,guo2016pac}.
%In general, the computational barriers for POMDPs are much more challenging
%to overcome than the statistical barriers. Nevertheless, we turn to this computational challenge in the next section.

% \qinghua{do we need to mention \cite{azizzadenesheli2016reinforcement,guo2016pac} are also computationally inefficient here?}
%% Algorithm \ref{alg:UCB} is not guaranteed to be
%% computationally efficient as the computation of optimistic policy
%% (Line \ref{line:greedy}) may not be done in polynomial time.

% \chijin{Comment on the computational complexity, that the max step in general can't be done in polynomial time.}

% \subsection{Computationally efficient learning of POMDPs with deterministic transition.}

\section{Results for POMDPs with Deterministic Transition}
In this section, we complement our main result by investigating the class
of POMDPs with deterministic transitions, where both computational and
statistical efficiency can be achieved.  We say a POMDP is of
\emph{deterministic transition} if both its transition and initial
distribution are deterministic, i.e, if the entries of matrices
$\{\T_h\}_h$ and vector $\mu_1$ are either $0$ or $1$. We remark that
while deterministic dynamics avoids computational barriers, it does
not mitigate the need for exploration.

Instead of Assumption~\ref{assump:POMDP}, for the deterministic transition case, we require that the columns of the observation matrices $\O_h$ are well-separated.
\begin{assumption}\label{assump:diverse_O}
For any $h\in [H]$, $\min_{s\neq s'}\norm{\O_h(\cdot|s) - \O_h(\cdot|s')} \ge \xi$.
\end{assumption}

Assumption \ref{assump:diverse_O} guarantees that observation distributions for different states are sufficiently different, by at least  $\xi$ in Euclidean norm.
It does not require that the POMDP is undercomplete, and, in fact, is strictly weaker than Assumption~\ref{assump:POMDP}.
%% being undercomplete. It can be further shown that Assumption \ref{assump:diverse_O} is strictly weaker than Assumption \ref{assump:POMDP}. 
In particular, for undercomplete models, $\min_{s\neq s'}\norm{\O_h(\cdot|s) - \O_h(\cdot|s')} \ge \sqrt{2}\sigma_{\min}(\O_h)$, and so Assumption \ref{assump:POMDP} implies Assumption \ref{assump:diverse_O} for $\xi = \sqrt{2} \alpha$.

Leveraging deterministic transitions, we can design a specialized algorithm (Algorithm \ref{alg:deterministic} in the appendix) that learns an $\varepsilon$-optimal policy using polynomially many samples and in polynomial time. We present the formal theorem here, and refer readers to Appendix \ref{app:determinisitc} for more details.
 % on the algorithm and the proofs of theorem.

\begin{restatable}{theorem}{DetTheorem}
\label{thm:deterministic}
For any $p \in (0, 1]$, there exists an algorithm such that for any deterministic transition POMDP satisfying Assumption \ref{assump:diverse_O}, within $\bigO\left( H^2SA \log(HSA/p)/(\min\{\varepsilon/(\sqrt{O}H),\xi\})^2 \right)$ samples and computations, the output policy of the algorithm is $\varepsilon$-optimal with probability at least $1-p$.
%  with deterministic transition $\T$ and initial $\mu$
% The statistical and computational complexity of Algorithm \ref{alg3} are both upper bounded by 
% . Let $\pihat$ be the optimal policy 
% for $(\hat{\mu}_0,\That,\rhat)$,
% with probability at least $1-\delta$, we have
% $R^\star - R^{\pihat} \le \eps$.
\end{restatable}

%!TEX root = main.tex

\section{Analysis Overview}\label{def:tau}

In this section, we provide an overview of the proof of our main result---Theorem \ref{thm:main}. Please refer to Appendix \ref{app:analysis} for the full proof.

We start our analysis by noticing that the output policy $\hat{\pi}$ of Algorithm \ref{alg:UCB} is uniformly sampled from $\{\pi_k\}_{k=1}^K$ computed in the algorithm. If we can show that
\begin{equation}\label{eq:objective}
\textstyle (1/K)\sum_{k=1}^K V^\star - V^{\pi_k}  \le \varepsilon/10,
\end{equation}
then at least a $2/3$ fraction of the policies in $\{\pi_k\}_{k=1}^K$ must be $\varepsilon$-optimal, and uniform sampling would find such a policy with probability at least $2/3$. Therefore, our proof 
focuses on achieving~\eqref{eq:objective}.
%% revolves around proving the averaged suboptimality in value \eqref{eq:objective} is small in our algorithm.

%% For simplicity of presentation, 
We begin by conditioning on the event that for each iteration $k$, our constructed confidence set $\Theta_k$ in fact contains the true parameters $\theta^\star = (\T,\O,\mu_1)$ of the POMDP.
%% we start our proof overview conditioning on that our constructed confidence set $\Theta_k$ in each iteration $k\in [K]$ in fact contains the true parameter $\theta^\star = (\T, \O, \mu_1)$ of the POMDP. 
This holds with high probability and is achieved by setting the widths $\beta_k$ and $\gamma_k$ appropriately (see Lemma~\ref{prop:realizability} in the appendix). 
%% Please see Lemma \ref{prop:realizability} for the formal statement and proofs.

\subsection{Bounding suboptimality in value by error in density estimation} 
% \paragraph{Step 1: Bounding suboptimality in value by error in density estimation.} 
Line \ref{line:greedy} of Algorithm \ref{alg:UCB} computes the greedy policy $\pi_k \leftarrow \argmax_{\pi} \max_{\hat{\theta} \in \Theta_k} V^\pi(\hat{\theta})$ with respect to the current confidence set $\Theta_k$. 
Let $\theta_k$ denote the maximizing model parameters in the $k$-th iteration. 
%% We denote $\theta_k$ as this maximizer of the parameter in $\Theta_k$ in $k$th iteration. 
As $(\pi_k, \theta_k)$ are optimistic, we have $V^\star \equiv
V^\star(\theta^\star) \le V^{\pi_k}(\theta_k)$ for all $k\in
[K]$. Thus, for any $k \in [K]$:
\begin{equation}\label{eq:VtoD}
V^\star - V^{\pi_k} \le V^{\pi_k}(\theta_k) - V^{\pi_k}(\theta^\star)
\le H \sum_{o_H, \ldots, o_1} |\Pr^{\pi_k}_{\theta_k}(o_H, \ldots, o_1) - \Pr^{\pi_k}_{\theta^\star}(o_H, \ldots, o_1)|,
\end{equation}
%% \begin{equation}\label{eq:VtoD}
%% \sum_{k=1}^K V^\star - V^{\pi_k} \le \sum_{k=1}^K V^{\pi_k}(\theta_k) - V^{\pi_k}(\theta^\star)
%% \le H \sum_{k=1}^K \sum_{o_H, \ldots, o_1} |\Pr^{\pi_k}_{\theta_k}(o_H, \ldots, o_1) - \Pr^{\pi_k}_{\theta^\star}(o_H, \ldots, o_1)|
%% \end{equation}
where $\Pr^{\pi}_\theta$ denotes the probability measure over observations under policy $\pi$ for POMDP with parameters $\theta$. 
The second inequality holds because the cumulative reward is a function of observations $(o_H, \ldots, o_1)$ and is upper bounded by $H$. 
This upper bounds the suboptimality in value by the total variation distance between the $H$-step observation distributions.

Next, note that we can always choose the greedy policy $\pi_k$ to be deterministic, i.e., the probability to take any action given a history is either $0$ or $1$. 
This allows us to define the following set for any deterministic policy $\pi$:
\begin{equation*}
\Gamma(\pi, H) := \{\tau_H = (o_H, \ldots, a_1, o_1) ~|~ \pi(a_{H-1}, \ldots, a_1 | o_H, \ldots, o_1) = 1\}.
\end{equation*}
In words, $\Gamma(\pi, H)$ is a set of all the observation and action sequences of length $H$ that could occur under the $\pi$. 
For any deterministic policy $\pi$, there is a one-to-one correspondence between $\cO^H$ and $\Gamma(\pi, H)$ and moreover, for any sequence $\tau_H = (o_H, \ldots, a_1, o_1) \in \Gamma(\pi, H)$, we have:
\begin{equation}\label{eq:justification2}
p(\tau_H; \theta) := \Pr_{\theta}(o_H, \ldots, o_1|a_{H-1}, \ldots, a_1) = \Pr^\pi_{\theta}(o_H, \ldots, o_1).
\end{equation}
The derivation of equation \eqref{eq:justification2}
 can be found in Appendix \ref{app:justification2}.
Combining this with~\eqref{eq:VtoD} and summing over all episodes, we conclude that:
\begin{equation*}
\sum_{k=1}^K (V^\star - V^{\pi_k}) \le H \sum_{k=1}^K \sum_{\tau_H \in \Gamma(\pi_k, H)} |p(\tau_H; \theta_k) - p(\tau_H; \theta^\star)|.
\end{equation*}
This upper bounds the suboptimality in value by errors in estimating the conditional probabilities.

% \newpage

% , and any observation sequence $(o_H, \ldots, o_1)$, there exist a unique sequence of action $(a_{H-1}, \ldots, a_1)$ so that $\tau_H = (o_H, \ldots, a_1, o_1)$ is inside $\Gamma(\pi)$.

% For any sequence $\tau_H$ of length $H$, we denote $p(\tau_H; \theta) := \Pr_{\theta}(o_H, \ldots, o_1|a_{H-1}, \ldots, a_1)$ to be the conditional probability of the sequence for POMDP with parameter $\theta$.

% % Step 1. By UCB, suboptimality is upper bounded by error in density estimation.

% \begin{align}
%     \sum_{k=1}^{K} \left(V^{\pi^\star}(\theta^\star)  - V^{\pi_k}(\theta^\star)  \right) 
%     \le & \sum_{k=1}^{K} \left( V^{\pi_k}(\theta_k) - V^{\pi_k}(\theta^\star) \right) \nn \\
%     \le & H \sum_{k=1}^{K} \sum_{\tau_{H-1}\in \Gamma(\pi_k,{H-1})} \| \b(\tau_{H-1};\theta) -\b(\tau_{H-1};\theta_k)\|_1.
% \end{align}

% \begin{align*}
% \sum_{k=1}^K (V^\star - V^{\pi_k}) \le& \sum_{k=1}^K (V_{\theta_k}^{\pi_k} - V_{\theta^*}^{\pi_k})
% \le H \sum_{k=1}^K \sum_{o_H, \ldots, o_1} |\Pr^{\pi_k}_{\theta_k}(o_H, \ldots, o_1) - \Pr^{\pi_k}_{\theta^\star}(o_H, \ldots, o_1)|\\
% =& H \sum_{k=1}^K \sum_{o_H, a_H\ldots, o_1, a_1 \in \Gamma(\pi_k)}|\Pr_{\theta_k}(o_H, \ldots, o_1|a_H, \ldots, a_1) - \Pr_{\theta^\star}(o_H, \ldots, o_1|a_H, \ldots, a_1)|
% \end{align*}

\subsection{Bounding error in density estimation by error in estimating operators} 
% \paragraph{Step 2: Bounding error in density estimation by error in estimating operators.} 
%% There comes to the major advantage of using OOM representation in 
For the next step, we leverage the OOM representation to bound the difference between the conditional probabilities $p(\tau_H; \theta_k)$ and $p(\tau_H; \theta^\star)$.
Recall that from \eqref{eq:OOM_pb}, the conditional probability can be written as a product of the observable operators for each step and $\b_0$. 
Therefore, for any two parameters $\hat{\theta}$ and $\theta$, we have following relation for any sequence $\tau_H = (o_H, \ldots, a_1, o_1)$:
\begin{align*}
&p(\tau_H; \hat{\theta}) - p(\tau_H; \theta) = \e_{o_H}\trans \cdot \B_{H-1}(a_{H-1}, o_{H-1};\hat{\theta}) \cdots \B_1(a_1, o_1; \hat{\theta}) \cdot [\b_0(\hat{\theta}) - \b_0(\theta)]\\
& + \sum_{h=1}^{H-1}\e_{o_H}\trans \cdot \B_{H-1}(a_{H-1}, o_{H-1};\hat{\theta}) \cdots [\B_h(a_h, o_h; \hat{\theta}) - \B_h(a_h, o_h; \theta)] \cdots \B_1(a_1, o_1; \theta) \cdot \b_0(\theta).
\end{align*}
This relates the difference $p(\tau_H; \hat{\theta}) - p(\tau_H; \theta)$ to the differences in operators and $\b_0$. 
Formally, with further relaxation and summation over all sequence in $\Gamma(\pi, H)$, we have the following lemma (also see Lemma~\ref{prop:subopt-bound} in Appendix~\ref{app:analysis}).
%% , which is a direct consequence of Proposition \ref{prop:subopt-bound} in Appendix \ref{app:analysis}

% Step 2. Decomposition from difference in density to difference in operator.

\begin{lemma}\label{lem:p_bound}
Given a deterministic policy $\pi$ and two sets of undercomplete POMDP parameters $\theta = (\O,\T,{\mu}_1)$ and $\thetahat = (\Ohat,\That,\hat{\mu}_1)$ with $\sigma_{\min}(\Ohat)\ge\alpha$, we have
\begin{align}
    & \sum_{\mathclap{\tau_H\in \Gamma(\pi,H)}} | p(\tau_H;\thetahat) - p(\tau_H;\theta)|  \le 
    \frac{\sqrt{S} }{\alpha} \left( \norm{\b_0(\hat{\theta})- \b_0(\theta)}_1  
    +  \sum_{\mathclap{(a,o)\in  \cA\times \cO}} \norm{ [\B_1(a,o;\thetahat) - \B_1(a,o;\theta)]\b_0(\theta)}_1 \right. \nonumber\\
& \hspace{2ex} +  \left.\sum_{h=2}^{H-1}\sum_{(a,\at,o)\in\cA^2\times\cO}
     \sum_{s=1}^{S} \left\| \left(\B_h(a,o;\thetahat) -\B_h(a,o;\theta)\right)  \O_{h} \T_{h-1}(\at) \e_s \right\|_1
     \Pr_\theta^\pi(s_{h-1}=s)\right). \label{eq:p_bound}
\end{align}
\end{lemma}

This lemma suggests that if we could estimate the operators
accurately, we would have small value sub-optimality. However,
Assumption \ref{assump:POMDP} is not sufficient for parameter
recovery.
%% We note our Assumption \ref{assump:POMDP} is not sufficient for parameter recovery. 
It is possible that in some step $h$, there exists a state $s_h$ that can be reached with only very small probability no matter what policy is used. 
Since we cannot collect many samples from $s_h$, it is not possible to estimate the corresponding component in the operator $B_h$. 
%% In these settings, it is no longer possible to recover the corresponding component in operator $\B_h$ within a small number of samples. 
In other words, we cannot hope to make $\norm{\B_h(a,o;\thetahat) -\B_h(a,o;\theta)}_1$ small in our setting.
%% , where exploration is important.

To proceed, it is crucial to observe that the third term on the RHS of
\eqref{eq:p_bound}, is in fact the operator error $\B_h(a,o;\thetahat)
-\B_h(a,o;\theta)$ projected onto the direction $ \O_{h} \T_{h-1}(\at)
\e_s $ and additionally reweighted by the probability of visiting
state $s$ in step $h-1$.  Therefore, if $s$ is hard to reach, the
weighting probability will be very small, which means that even though
we cannot estimate $\B_h(a,o;\theta)$ accurately in the corresponding
direction, it has a negligible contribution to the density estimation
error (LHS of \eqref{eq:p_bound}).

%% instead, only depends on the error of $\B_h(a,o;\thetahat) -\B_h(a,o;\theta)$ projected onto the direction $ \O_{h} \T_{h-1}(\at) \e_s $, which is further reweighted by the probability to reach state $s$ in $(h-1)$th step. 
%% Therefore, for hard to reach state, the weighting probability is very small. Despite the fact that we can not estimate $\B_h(a,o;\theta)$ accurately in the corresponding direction of that state, it has negligible contribution to the final error in density estimation (LHS of \eqref{eq:p_bound}).

% \begin{align*}
%      \sum_{\tau_h\in \Gamma(\pi,h)} \| b(\tau_h) -\bt(\tau_h)\|_1 
%    \le & \sqrt{S} \|\Ot\sdinv\|_2 \left\|  b_0 - \bt_0 \right\|_1  + \sqrt{S} \|\Ot\sdinv\|_2  \sum_{(x,a)\in  \Omega \times \cA}
%      \left\|  \left (\At_{x,a} - 
%      A_{x,a}\right)b_0 \right\|_1\\
%  +   \sqrt{S} \|\Ot\sdinv\|_2 
%      &\sum_{(a,a\pr,x)\in\cA^2\times\Omega}
%      \sum_{s=1}^{S} \| \left(\At_{x,a} -A_{x,a}\right)(\O \T_{a\pr})_s \|_1
%    \sum_{j=2}^{h}  \Pr^\pi(s_{j-1}=s).
% \end{align*}

\subsection{Bounding error in estimating operators by OOM-UCB algorithm} 
% \paragraph{Step 3: Bounding error in estimating operators by OOM-UCB algorithm} 
By Lemma \ref{lem:p_bound}, we only need to bound the error in operators reweighted by visitation probability. 
This is achieved by a careful design of the confidence sets in the OOM-UCB algorithm. This construction is based on the method of moments, which heavily exploits the undercompleteness of the POMDP. To showcase the main idea, we focus on bounding the third term on the RHS of \eqref{eq:p_bound}.
 % as the bound on the other terms is obtained in a similar fashion.

Consider a fixed $(o, a, \tilde{a})$ tuple, a fixed step $h\in[H]$, and a fixed iteration $k \in [K]$. 
We define moment matrices $\mat{P}_h(a, \tilde{a}), \mat{Q}_h(o, a, \tilde{a}) \in \R^{O\times O}$ as in \eqref{eq:PnQ} for distribution on $s_{h-1}$ that equals $(1/k)\cdot \sum_{t=1}^{k}\Pr_{\theta^\star}^{\pi_{t}}(s_{h-1} = \cdot) $. 
We also denote $\hat{\mat{P}}_h(a, \tilde{a}) = \bP_h(a, \tilde{a})/k, \hat{\mat{Q}}_h(o, a, \tilde{a}) = \bQ_h(o, a, \tilde{a})/k$ for $\bP_h, \bQ_h$ matrices \emph{after the update} in the $k$-th iteration of Algorithm \ref{alg:UCB}. 
By martingale concentration, it is not hard to show that with high probability:
\begin{equation*}
\norm{\mat{P}_h(a, \tilde{a}) - \hat{\mat{P}}_h(a, \tilde{a})}_F \le \tlO(1/\sqrt{k}), \qquad \norm{\mat{Q}_h(o, a, \tilde{a}) - \hat{\mat{Q}}_h(o, a, \tilde{a})}_F \le \tlO(1/\sqrt{k}).
\end{equation*}
Additionally, we can show that for the true operator and the true moments, we have $\B_h(a, o; \theta^\star)\mat{P}_h(a, \tilde{a}) = \mat{Q}_h(o, a, \tilde{a})$. 
Meanwhile, by the construction of our confidence set $\Theta_{k+1}$, we know that for any $\hat{\theta} \in \Theta_{k+1}$, we have
\begin{equation*}
\norm{\B_h(a, o; \hat{\theta}) \hat{\mat{P}}_{h} (a, \tilde{a}) -\hat{\mat{Q}}_h (o, a, \tilde{a})}_F \le \gamma_k/k.
\end{equation*}
% \akshay{Should this be $1/\sqrt{k}$ above?}
Combining all relations above, we see that $\B_h(a, o; \hat{\theta})$ is accurate in the directions spanned by $\mat{P}_h(a, \tilde{a})$, which, by definition, are directions frequently visited by the previous policies $\{\pi_t\}_{t=1}^k$. Formally, we have the following lemma, which allows us to bound the third term on the RHS of~\eqref{eq:p_bound} using the algebraic transformation in Lemma \ref{lem:root-regret}.
% , which bound the difference between $\B_h(a, o; \hat{\theta})$ and $\B_h(a, o; \theta^\star)$.
\begin{lemma}\label{lem:B_bound}
With probability at least $1-{\delta}$, for all $k\in[K]$, for any $\thetahat=(\Ohat,\That,\hat{\mu}_1)\in\Theta_{k+1}$ and $(o,a,\at,h)\in\cO \times \cA^2\times\{2,\ldots,H-1\}$, and $\iota = \log(KAOH/\delta)$, we have	
\begin{equation*}
 \sum_{s=1}^{S} \left\| \left(\B_h(a,o;\thetahat) -\B_h(a,o;\theta^\star)\right)  \O_{h} \T_{h-1}(\at) \e_s \right\|_1
    \sum_{t=1}^{k} \Pr^{\pi_t}_{\theta^\star}(s_{h-1}=s) 
     \le \bigO\left(\sqrt{\frac{k S^2 O\iota}{\alpha^4}}\right).
\end{equation*}
\end{lemma}
% Similar arguments for the other two terms in~\eqref{eq:p_bound} lead to the final guarantee. 

%% Lemma \ref{lem:B_bound} is a consequence of Proposition
%% \ref{prop:confidence-set-ucb} in Appendix \ref{app:analysis}.  Finally,
%% by algebraic transformation (Proposition \ref{lem:root-regret}), we
%% use Lemma \ref{lem:B_bound} to bound the third term in~\eqref{eq:p_bound}, which concludes the proof.

% \akshay{I think what is missing here is an explanation about how using
  % $\hat{\mat{P}}$ in the constraint means we have good error on the
  % directions we have visited a lot? It is not coming across that
  % $\B_h\mat{P}_h$ being accurate is what we really want.}

% \begin{align*}
%     &\sum_{s=1}^{S} \left\| \left(\At_{x,a} -A_{x,a}\right)  \left(\O \T_{a\pr}\right)_s \right\|_1
%    \sum_{t=1}^{k-1} \sum_{h=2}^{H}  \Pr^{\pi_t} (s_{h-1}=s) \\
% \le  & c_2 \sqrt{kHSO\log(AHK/\delta)}\|  {\O\sdinv}\|_2 \left(    \|\At_{x,a}\|_2+\|   A_{x,a}\|_2 + 1\right),
% \end{align*}

%!TEX root = main.tex

\section{Conclusion}
In this paper, we give a sample efficient algorithm for reinforcement
learning in undercomplete POMDPs. Our results leverage a connection to
the observable operator model and employ a refined error analysis. To
our knowledge, this gives the first provably efficient algorithm for
strategic exploration in partially observable environments.

%!TEX root = main.tex

%\section*{Broader Impact}
%
%As this is a theoretical contribution, we do not envision that our
%direct results will have a tangible societal impact.  Our broader line
%of inquiry could impact a line of thinking in a way that provides
%additional means to provide confidence intervals relevant for
%planning and learning.  There is an increasing needs for applications
%to understand planning under uncertainty in the broader context of
%safety and reliability, and POMDPs provide one potential framework. 

% There are numerous applications in which uncertainty is a central component in planning. This broader line of inquiry could impact this line of thinking in way that emphases how strategic exploration should be coupled with uncertainty not only in the model but in the hidden state estimation. Any applications clearly need to understand these in the context of safety and reliability due to mis-specified models.

\section*{Acknowledgements}

This work was supported by Microsoft and Princeton University.
S.K. gratefully acknowledges funding from the ONR award N00014-18-1-2247, NSF Awards CCF-1703574 and CCF-1740551.

\bibliographystyle{abbrvnat}
\bibliography{ref}
%\input{planning}
%\newpage
% \input{operator_model}
% \newpage
% \input{algorithm_rmax}
% \newpage
\newpage
\appendix
\section{Notation}\label{appendix-notation}
Below, we introduce some notations that will be used in appendices.

\begin{table}[H]
\renewcommand\arraystretch{1.5}
%\centering
	\begin{tabular}{|P{2.5cm}||P{10cm}|}
 \hline
 %\multicolumn{2}{|c|}{Notation List } \\
 %\hline
 notation  & definition\\
 \hline \hline
 $\bp^k$ & value of $\bp$ \emph{after the update} in the $k^{\rm th}$  iteration of Algorithm \ref{alg:UCB}\\
 $\bP_h^k(a,\at)$ & value of $\bP_h(a,\at)$ \emph{after the update} in the $k^{\rm th}$  iteration of Algorithm \ref{alg:UCB}\\
 $\bQ_h^k(o,a,\at)$ & value of $\bQ_h(o,a,\at)$ \emph{after the update} in the $k^{\rm th}$  iteration of Algorithm \ref{alg:UCB}\\
 $\theta$ & a parameter triple $(\T,\O,\mu_1)$ of a POMDP\\
 $\theta^\star$ & the groundtruth POMDP parameter triple\\
 POMDP$(\theta)$ & $\pomdp(H, \cS, \cA, \cO, \T, \O, r, \mu_1)$\\
 $\tau_h$\footnotemark 
 &  a length-$h$ trajectory: $\tau_h=[a_h,o_h,\ldots,a_1,o_1]\in(\cA\times\cO)^h$\\
 $\Gamma(\pi,h)$\footnotemark & $\{\tau_h = (a_h,o_h, \ldots, a_1, o_1) ~|~ \pi(a_{h}, \ldots, a_1 | o_h, \ldots, o_1) = 1\}$.\\
 $\b(\tau_{h};\theta)$ &  $\B_{h}(a_{h}, o_{h};\theta) \cdots \B_1(a_1, o_1;\theta) \cdot \b_0(\theta)$\\
 $\Pr^{\pi}_\theta ( s_{h} = s)$  & probability of visiting state $s$ at $h^{\rm th}$ step when executing policy $\pi$ on POMDP($\theta$)\\
 \hline
 $\one(x=y)$ & equal to $1$ if $x=y$ and $0$ otherwise.\\
 $\e_{o}$ & an $O$-dimensional vector with $(\e_{o})_i = \one(o=i)$\\
 $(\X)_o$ & the $o^{\rm th}$ column of matrix $\X$\\
 $\I_{n}$  & $n\times n$ identity matrix\\
  $\Cpoly$ & $\poly(S,O,A,H,{1}/ {\alpha},\log({1}/ {\delta}))$\\
  $\iota$ & $\log(AOHK/\delta)$\\
 \hline
\end{tabular}
\end{table}
\footnotetext[3]{Note that this definition is \emph{different} from the one used in Section \ref{def:tau}, where $\tau_h = [o_h,\ldots,a_1,o_1]\in\cO\times (\cA\times\cO)^{h-1}$ does not include the action $a_h$ at $h^{\rm th}$ step.}
\footnotetext[4]{WLOG, all the polices considered in this paper are \emph{deterministic}. Also note that the trajectory in $\Gamma(\pi,h)$ contains $a_h$, which is \emph{different} from the definition in Section \ref{def:tau}}

Let $\x\in\R^{n_x}$, $\y \in \R^{n_y}$ and $\z \in \R^{n_z}$. We denote by $\x \otimes \y \otimes \z$ the tensor product of vectors $\x$, $\y$ and $\z$, an $n_x\times n_y \times n_z$ tensor with $(i,j,k)^{\rm th}$ entry equal to $\x_i \y_j \z_k$. 
Let $\X \in \R^{n_X\times m}$, $\Y\in \R^{n_Y\times m}$ and $\mathbf{Z}\in\R^{n_Z\times m}$. We generalize the notation of tensor product to matrices by defining $\X \otimes \Y \otimes \mathbf{Z}= \sum_{l=1}^{m} (\X)_l \otimes (\Y)_l \otimes (\mathbf{Z})_l$, which is an $n_X\times n_Y \times n_Z$ tensor with $(i,j,k)^{\rm th}$ entry equal to $\sum_{l=1}^{m} \X_{il} \Y_{jl} \mathbf{Z}_{kl}$. 

Let $X$ be a random variable taking value in $[m]$, we denote by $\Pr(X=\cdot)$ an $m$-dimensional vector whose $i^{\rm th}$ entry is $\Pr(X=i)$.

\newpage

\section{Proof of Hardness Results}
\label{app:hardness}
{
\newcommand{\ahstar}{a_h^{\star}}

The hard examples constructed below are variants of the ones used in \cite{krishnamurthy2016pac}.

\LowerBoundOvercomplete*
\begin{proof}
	Consider the following $H$-step nonstationary POMDP: 
	\begin{enumerate}
		\item STATE There are four states: two good states $g_{1}$ and $g_{2}$ and two bad states $b_{1}$ and $b_{2}$. The initial state is picked  uniformly at random.
		\item OBSERVATION There are only two different observations $u_1$ and $u_2$. At step $h\in[H-1]$, we always observe $u_1$ at $g_{1}$ and $b_{1}$, and  observe $u_2$ at $g_{2}$ and $b_{2}$. At step $H$, we always observe $u_1$ at good states and  $u_2$ at bad states. It's direct to verify $\sigma_{\min}(\O_h)=1$ for all $h\in[H]$.
		\item REWARD There is no reward at the fist $H-1$ steps (i.e. $r_h=0$ for all $h\in[H-1]$). At step $H$, we receive reward $1$ if we observe $u_1$ and no reward otherwise (i.e. $r_H(o)=\one(o=u_1)$).
		\item TRANSITION There is one good action $\ahstar$ and $A-1$ bad actions for each $h\in[H-1]$. 
		At step $h\in[H-1]$, suppose we are at a good state ($g_{1}$ or $g_{2}$), then we will transfer to $g_{1}$ or $g_{2}$  uniformly at random if we take $\ahstar$ and otherwise transfer to $b_{1}$ or $b_{2}$  uniformly at random. In contrast, if we are at a bad state ($b_{1}$ or $b_{2}$), we will always transfer to $b_{1}$ or $b_{2}$  uniformly at random no matter what action we take. Note that two good (bad) states are equivalent in terms of transition.
	\end{enumerate}
	We have the following key observations:
	\begin{enumerate}
		\item Once we are at bad states, we always stay at bad states.
		\item We have 
		\begin{align*}
			&\Pr(o_{1:H-1}=z\mid a_{1:H-1},o_H)=\frac{1}{2^{H-1}} 
			\\ &\mbox{for any } z\in \{u_1,u_2\}^{H-1} \mbox{ and } (a_{1:H-1},o_H)\in [A]^{H-1}\times\{u_1,u_2\}
		\end{align*}
		Therefore, the observations at the first $H-1$ steps provide no information about the underlying transition. The only useful information is the last observation $o_H$ which tells us whether we end in good states or not.
		\item The optimal policy is unique and is to execute the good action sequence $(a_1^\star,\ldots,a_{H-1}^\star)$ regardless of the obervations.
	\end{enumerate}
	Based on the observations above, this is equivalent to a multi-arm bandits problem with $A^{H-1}$ arms. Therefore, we cannot do better than Brute-force search, which has sample complexity at least $\Omega(A^{H-1})$.
\end{proof}

\LowerBoundUndercomplet*
\begin{proof}
	We continue to use the POMDP constructed in Proposition~\ref{prop:hardness1} and slightly modify it by splitting $u_2$ into another $4$ different observations $\{q_1,q_2,q_3,q_4\}$, so in the new POMDP ($O=5> S=4$), we will observe a $q_i$ picked  uniformly at random from $\{q_1,q_2,q_3,q_4\}$ when we are 'supposed' to observe $u_2$.
	 It's easy to see the modification does not change its hardness.
	\end{proof}

}

\section{Analysis of OMM-UCB} \label{app:analysis}

	Throughout the proof, we use $\tau_h$ to denote a length-$h$ trajectory: $[a_h,o_h,\ldots,a_1,o_1]\in(\cA\times\cO)^h$. Note that this definition is \emph{different} from the one used in Section \ref{def:tau}, where $\tau_h = [o_h,\ldots,a_1,o_1]\in\cO\times (\cA\times\cO)^{h-1}$ does not include the action $a_h$ at $h^{\rm th}$ step.
	Besides, we define 
	$$\Gamma(\pi,h) = \{\tau_h = (a_h,o_h, \ldots, a_1, o_1) ~|~ \pi(a_{h}, \ldots, a_1 | o_h, \ldots, o_1) = 1\},$$
	 which is also \emph{different} from the definition in Section \ref{def:tau} where $a_h$ is not included.
	
	Please refer to Appendix \ref{appendix-notation} for definitions of frequently used notations.

\subsection{Bounding the error in belief states}
In this subsection, we will bound the error 
in (unnormalized) belief states, i.e., $\b(\tau_h;\theta) -\b(\tau_h;\thetahat)$ by the error in  operators reweighed by the probability distribution of visited states.

We start by proving the following lemma that helps us decompose the error in belief states inductively.

\begin{lemma}\label{lemma:belief-induction}
Given a deterministic policy $\pi$ and two set of POMDP parameters $\thetahat = (\Ohat,\That,\hat{\mu}_1)$ and $\theta = (\O,\T,{\mu}_1)$,
for all $h\ge1$ and $\X\in\{\I_O,\hat{\O}_{h+1}^\dagger\}$, we have
\begin{align*}
\sum_{\tau_h\in \Gamma(\pi,h)} \left\| \X\left(\b(\tau_h;\theta) -\b(\tau_h;\thetahat)\right)\right\|_1
\le 
&\sum_{\tau_{h-1}\in \Gamma(\pi,h-1)}
\left\|\hat{\O}_h^\dagger\left(  \b(\tau_{h-1};\theta)-  \b(\tau_{h-1};\thetahat)\right)\right\|_1 \\
 + \sum_{\tau_h\in \Gamma(\pi,h)} &\left\| \X\left(\B_h(a_h,o_h;\thetahat) -\B_h(a_h,o_h;\theta)\right)\b(\tau_{h-1};\theta)\right\|_1.
\end{align*}
\end{lemma}
\begin{proof}
By the definition of $\b(\tau_h;\theta)$ and $\b(\tau_h;\thetahat)$, 
\begin{align*}
&\sum_{\tau_h\in \Gamma(\pi,h)} \| \X\left(\b(\tau_h;\theta) -\b(\tau_h;\thetahat)\right)\|_1\\
 =   &\sum_{\tau_h\in \Gamma(\pi,h)}\| \X\left(\B_h(a_h,o_h;\theta)\b(\tau_{h-1};\theta)-\B_h(a_h,o_h;\thetahat)\b(\tau_{h-1};\thetahat)\right)\|_1  \\
    \le & 
   \sum_{\tau_h\in \Gamma(\pi,h)} \| \X\B_h(a_h,o_h;\thetahat)\left(\b(\tau_{h-1};\theta)-\b(\tau_{h-1};\thetahat)\right)\|_1 \\
    + &\sum_{\tau_h\in \Gamma(\pi,h)}\| \X\left(\B_h(a_h,o_h;\thetahat) -\B_h(a_h,o_h;\theta)\right)\b(\tau_{h-1};\theta)\|_1.
\end{align*}
The first term can be bounded as following,
\begin{align*}
    &\sum_{\tau_h\in \Gamma(\pi,h)} 
    \| \X \B_h(a_h,o_h;\thetahat)(  \b(\tau_{h-1};\theta)-  \b(\tau_{h-1};\thetahat))\|_1 \\
 = & 
    \sum_{\tau_h\in \Gamma(\pi,h)}
    \|\X \hat{\O}_{h+1} \That_h(a_{h})\diag(\hat{\O}_h(o_h\mid \cdot)) \hat{\O}_h^\dagger\left(  \b(\tau_{h-1};\theta)-  \b(\tau_{h-1};\thetahat)\right)\|_1 \\
\le & \sum_{\tau_h\in \Gamma(\pi,h)} \sum_i 
\left\| \left(\X \hat{\O}_{h+1}\That_h(a_{h})\diag(\hat{\O}_{h}(o_h\mid \cdot))\right)_i\right\|_1 \left|\left(\hat{\O}_h^\dagger\left(  \b(\tau_{h-1};\theta)-  \b(\tau_{h-1};\thetahat)\right)\right)_i\right|\\
=& \sum_{\tau_h\in \Gamma(\pi,h)}\sum_i 
\left\| \left(\X \hat{\O}_{h+1}\That_h(a_{h})\right)_i\right\|_1 \hat{\O}_h(o_h \mid i) \left|\left(\hat{\O}_h^\dagger\left(  \b(\tau_{h-1};\theta)-  \b(\tau_{h-1};\thetahat)\right)\right)_i\right|\\
= & \sum_{\tau_h\in \Gamma(\pi,h)}\sum_i 
\hat{\O}_h(o_h \mid i) \left|\left(\hat{\O}_h^\dagger\left(  \b(\tau_{h-1};\theta)-  \b(\tau_{h-1};\thetahat)\right)\right)_i\right|,
\end{align*}
where the inequality is by triangle inequality, 
and the last identity follows from 
$\That_h(a_{h})$ (when $\X=\hat{\O}_{h+1}^\dagger$) and $\hat{\O}_{h+1} \That_h(a_{h})$ (when $\X=\I_O$) having columns with
$\ell_1$-norm equal to $1$.

As $\pi$ is deterministic, $a_h$ is unique given $\tau_{h-1}$ and $o_h$. Therefore,
\begin{align*}
    & \sum_{\tau_h\in \Gamma(\pi,h)}\sum_i 
\hat{\O}_h(o_h \mid i) \left|\left(\hat{\O}_h^\dagger\left(  \b(\tau_{h-1};\theta)-  \b(\tau_{h-1};\thetahat)\right)\right)_i\right|\\
= & \sum_{\tau_{h-1}\in \Gamma(\pi,h-1)} \sum_{o_h}\sum_i
\hat{\O}_h(o_h \mid i) \left|\left(\hat{\O}_h^\dagger\left(  \b(\tau_{h-1};\theta)-  \b(\tau_{h-1};\thetahat)\right)\right)_i\right|\\
= & \sum_{\tau_{h-1}\in \Gamma(\pi,h-1)} \sum_i\sum_{o_h}
\hat{\O}_h(o_h \mid i) \left|\left(\hat{\O}_h^\dagger\left(  \b(\tau_{h-1};\theta)-  \b(\tau_{h-1};\thetahat)\right)\right)_i\right|\\
= & \sum_{\tau_{h-1}\in \Gamma(\pi,h-1)} \sum_i\left|\left(\hat{\O}_h^\dagger\left(  \b(\tau_{h-1};\theta)-  \b(\tau_{h-1};\thetahat)\right)\right)_i\right|\\
= & \sum_{\tau_{h-1}\in \Gamma(\pi,h-1)} 
\left\|\hat{\O}_h^\dagger\left(  \b(\tau_{h-1};\theta)-  \b(\tau_{h-1};\thetahat)\right)\right\|_1,
\end{align*}
which completes the proof.
\end{proof}

By applying Lemma \ref{lemma:belief-induction}
inductively, we can bound the error in belief states 
by the projection of errors in operators on preceding belief states.

\begin{lemma}\label{lem:regret-decomp}
Given a deterministic policy $\pi$ and two sets of undercomplete POMDP parameters $\theta = (\O,\T,{\mu}_1)$ and $\thetahat = (\Ohat,\That,\hat{\mu}_1)$ with $\sigma_{\min}(\Ohat)\ge\alpha$,
for all $h\ge1$, we have
\begin{align*}
&      \sum_{\tau_h\in \Gamma(\pi,h)} \left\| \b(\tau_h;\theta) -\b(\tau_h;\thetahat)\right\|_1 \\
    \le & \frac{\sqrt{S} }{\alpha}\sum_{j=1}^{h} \sum_{\tau_{j} \in \Gamma(\pi,j)} 
     \left\| \left(\B_j(a_j,o_j;\thetahat) -\B_j(a_j,o_j;\theta)\right)\b(\tau_{j-1};\theta)\right\|_1
    + \frac{\sqrt{S} }{\alpha} \left\| \b_0(\theta)- \b_0(\thetahat) \right\|_1.
\end{align*}
\end{lemma}

\begin{proof}
	Invoking Lemma \ref{lemma:belief-induction} with $\X=\Ohat_{j+1}\sdinv$, we have
\begin{align}\label{eqApr3:1}
\sum_{\tau_j\in \Gamma(\pi,j)} \| &\Ohat_{j+1}\sdinv\left(\b(\tau_j;\theta) -\b(\tau_{j};\thetahat)\right)\|_1 \le 
 \sum_{\tau_{j-1}\in \Gamma(\pi,j-1)}
\left\|\Ohat_{j}^\dagger\left(  \b(\tau_{j-1};\theta)-  \b(\tau_{j-1};\thetahat)\right)\right\|_1 \nn \\
 &+ \sum_{\tau_j\in \Gamma(\pi,j)} \| \Ohat_{j+1}\sdinv\left(\B_j(a_j,o_j;\thetahat) -\B_j(a_j,o_j;\theta)\right)\b(\tau_{j-1};\theta)\|_1.
\end{align}
Summing \eqref{eqApr3:1} over $j=1,\ldots,h-1$, we obtain
\begin{align}
\label{apr6:0}
      &\sum_{\tau_{h-1}\in \Gamma(\pi,h-1)} \| \Ohat_h \sdinv\left(\b(\tau_{h-1};\theta) -\b(\tau_{h-1};\thetahat) \right)\|_1 \\
    \hspace{0ex} \le& \sum_{j=1}^{h-1} \sum_{\tau_{j} \in \Gamma(\pi,j)} 
    \left\| \Ohat_{j+1}\sdinv\left(\B_j(a_j,o_j;\thetahat) -\B_j(a_j,o_j;\theta)\right)\b(\tau_{j-1};\theta)\right\|_1\nn 
    +\left\| \Ohat_1\sdinv \left(\b_0(\theta)- \b_0(\thetahat) \right)\right\|_1.
\end{align}
Again, invoking Lemma \ref{lemma:belief-induction} with $\X=\I_O$  gives 
\begin{align}\label{apr6:1}
      \sum_{\tau_h\in \Gamma(\pi,h)} \| \b(\tau_h;\theta) -\b(\tau_h;\thetahat)\|_1 \nn 
    \le &\sum_{\tau_{h-1} \in \Gamma(\pi,h-1)} \| \Ohat_h\sdinv ( \b(\tau_{h-1};\theta) - \b(\tau_{h-1};\thetahat))\|_1\\
    + \sum_{\tau_{h}\in \Gamma(\pi,h)} &\| \left(\B_h(a_h,o_h;\thetahat) -\B_h(a_h,o_h;\theta)\right)\b(\tau_{h-1};\theta)\|_1.
\end{align}
Plugging \eqref{apr6:0} into \eqref{apr6:1},
and using the fact that $\| \Ohat_h \sdinv\|_{1\rightarrow1} \le \sqrt{S}\| \Ohat_h\sdinv\|_2\le \frac{\sqrt{S}}{\alpha}$ complete the proof.
\end{proof}

The following lemma bounds the projection of any vector on belief states by its projection on the product of the observation matrix and the transition matrix, reweighed by the visitation probability of states.

\begin{lemma}\label{lem:belief-pairdistribution}
For any deterministic  policy $\pi$,
\emph{fixed} $a_{h+1}\in\cA$, $\u \in \mathbb{R}^{O}$,
and $h\ge0$,  we have
\begin{align*}
     \sum_{o_{h+1}\in\cO} \sum_{\tau_h\in \Gamma(\pi,h)} 
     \left | \u \trans  \b([a_{h+1},o_{h+1},\tau_h];\theta) \right|
     \le \sum_{s=1}^{S}
    | \u \trans (\O_{h+2} \T_{h+1}(a_{h+1}))_s | \Pr^\pi_\theta(s_{h+1}=s).
\end{align*}
\end{lemma}
\begin{proof}
	
By definition, for any $[a_{h+1},o_{h+1},\tau_h]\in \cA 
\times \cO \times \Gamma(\pi,h)$,  
we have
$$
\b([a_{h+1},o_{h+1},\tau_h];\theta)=\O_{h+2}\T_{h+1}(a_{h+1}) \Pr^\pi_\theta(s_{h+1}=\cdot,[o_{h+1}, \tau_h]),
$$
where $\Pr^\pi_\theta(s_{h+1}=\cdot,[o_{h+1}, \tau_h])$ is an $s$-dimensional vector, whose $i^{\rm th}$ entry is equal to the probability of observing $[o_{h+1}, \tau_h]$ and reaching state $i$ at step $h+1$ when executing policy $\pi$ in $\pomdp(\theta)$.

Therefore,
\begin{align*}%\label{step2eq4}
    &\sum_{\tau_{h}\in \Gamma(\pi,h)} \sum_{o_{h+1}\in\cO} | \u \trans 
    \b([a_{h+1},o_{h+1},\tau_h];\theta)|\nn\\
    =& \sum_{\tau_{h}\in \Gamma(\pi,h)} \sum_{o_{h+1}\in\cO} | \u \trans 
    \O_{h+2}\T_{h+1}(a_{h+1}) \Pr^\pi_\theta(s_{h+1}=\cdot,[o_{h+1},\tau_h])|\nn\\
    \le & \sum_{\tau_{h}\in \Gamma(\pi,h)} \sum_{o_{h+1}\in\cO} \sum_{s=1}^{S}
    | \u \trans (\O_{h+2}\T_{h+1}(a_{h+1}))_s | \Pr^\pi_\theta(s_{h+1}=s,[o_{h+1},\tau_h])\nn\\
    = & \sum_{s=1}^{S}
    | \u \trans (\O_{h+2}\T_{h+1}(a_{h+1}))_s |\bigg( \sum_{\tau_{h}\in \Gamma(\pi,h)} \sum_{o_{h+1}\in\cO} \Pr^\pi_\theta(s_{h+1}=s,[o_{h+1},\tau_{h}])\bigg)\nn\\
    =& \sum_{s=1}^{S}
    | \u \trans (\O_{h+2}\T_{h+1}(a_{h+1}))_s | \Pr^\pi_\theta(s_{h+1}=s).\tag*\qedhere
    %\nn\\
    %=& \| \u \trans P_{h+2,h+1}(\pi,a) {\O\sdinv}\trans\|_1.
\end{align*}
\end{proof}

Combining Lemma \ref{lem:regret-decomp} and Lemma 
\ref{lem:belief-pairdistribution}, we obtain the target bound.

\begin{lemma}\label{prop:subopt-bound}
Given a deterministic policy $\pi$ and two sets of undercomplete POMDP parameters $\theta = (\O,\T,{\mu}_1)$ and $\thetahat = (\Ohat,\That,\hat{\mu}_1)$ with $\sigma_{\min}(\Ohat)\ge\alpha$,
for all $h\ge1$, we have
\begin{align*}
&     \sum_{\tau_h\in \Gamma(\pi,h)} \| \b(\tau_h;\theta) -\b(\tau_h;\thetahat)\|_1 \\
   \le & \frac{\sqrt{S} }{\alpha} \left\| \b_0(\theta)- \b_0(\thetahat) \right\|_1  
    + \frac{\sqrt{S} }{\alpha}  \sum_{(a,o)\in  \cA\times \cO}
     \left\|  \left (\B_1(a,o;\thetahat) - 
     \B_1(a,o;\theta)\right)\b_0(\theta)\right\|_1\\
& +   \frac{\sqrt{S} }{\alpha} 
     \sum_{j=2}^{h}\sum_{(a,\at,o)\in\cA^2\times\cO}
     \sum_{s=1}^{S} \left\| \left(\B_j(a,o;\thetahat) -\B_j(a,o;\theta)\right)  (\O_{j} \T_{j-1}(\at))_s \right\|_1
     \Pr^\pi_\theta(s_{j-1}=s).
\end{align*}
\end{lemma}
\begin{proof}
By Lemma \ref{lem:regret-decomp}, 
\begin{align}\label{step0}
      &\sum_{\tau_h\in \Gamma(\pi,h)} \| \b(\tau_h;\theta) -\b(\tau_h;\thetahat)\|_1 \nn \\
    \le & \frac{\sqrt{S} }{\alpha}\sum_{j=2}^{h} \sum_{\tau_{j} \in \Gamma(\pi,j)} 
     \left\| \left(\B_j(a_j,o_j;\thetahat) -\B_j(a_j,o_j;\theta)\right)\b(\tau_{j-1};\theta)\right\|_1\nn\\
    &+ \frac{\sqrt{S} }{\alpha} \sum_{\tau_1 \in  \Gamma(\pi,1)} \left\|  \left (\B_1(a_1,o_1;\thetahat) - 
     \B_1(a_1,o_1;\thetahat)\right)\b_0(\theta) \right\|_1 +  \frac{\sqrt{S} }{\alpha} \left\|\b_0(\theta)- \b_0(\thetahat)\right\|_1.
\end{align}

{\bf Bounding the first term:}
note that $\Gamma(\pi,j) \subseteq \Gamma(\pi,j-2)\times (\cO\times \cA)^2$, so we have
\begin{align}\label{step1eq1}
     & \sum_{\tau_{j}\in \Gamma(\pi,j)}\| \left(\B_j(a_j,o_j;\thetahat) -\B_j(a_j,o_j;\theta)\right)\b(\tau_{j-1};\theta)\|_1 \nn\\
\le & \sum_{\tau_{j-2}\in \Gamma(\pi,j-2)} \sum_{o_{j-1}\in \cO} \sum_{a_{j-1}\in\cA} \sum_{o_j\in\cO} \sum_{a_j\in\cA}\nn\\
& \|  \left(\B_j(a_j,o_j;\thetahat) -\B_j(a_j,o_j;\theta)\right)\b([a_{j-1},o_{j-1},\tau_{j-2}];\theta)\|_1 \nn\\
=& \sum_{(a_j,a_{j-1},o_{j})\in\cA^2\times\cO}\nn\\&\underbrace{\sum_{\tau_{j-2}\in \Gamma(\pi,j-2)} \sum_{o_{j-1}\in \cO}  \|  \left(\B_j(a_j,o_j;\thetahat) -\B_j(a_j,o_j;\theta)\right)\b([a_{j-1},o_{j-1},\tau_{j-2}];\theta)\|_1}_{\textstyle (\diamond)}.
\end{align}
We can bound $(\diamond)$ by Lemma \ref{lem:belief-pairdistribution} and obtain,
\begin{align}\label{step1eq2}
&  \sum_{\tau_{j}\in \Gamma(\pi,j)}\| \left(\B_j(a_j,o_j;\thetahat) -\B_j(a_j,o_j;\theta)\right)\b(\tau_{j-1};\theta)\|_1 \nn\\
     \le &   \sum_{(a_j,a_{j-1},o_{j})\in\cA^2\times\cO}
     \sum_{s=1}^{S} \| \left(\B_j(a_j,o_j;\thetahat) -\B_j(a_j,o_j;\theta)\right)  (\O_{j} \T_{j-1}(a_{j-1}))_s \|_1
     \Pr^\pi_\theta(s_{j-1}=s) \nn \\
     = &  \sum_{(a,\at,o)\in\cA^2\times\cO}
     \sum_{s=1}^{S} \| \left(\B_j(a,o;\thetahat) -\B_j(a,o;\theta)\right) (\O_{j} \T_{j-1}(\at))_s \|_1
     \Pr^\pi_\theta(s_{j-1}=s),
     \end{align}
     where the identity only changes the notations $(a_j,a_{j-1},o_{j})\rightarrow(a,\at,o)$ to make the expression cleaner. 

{\bf Bounding the second term:}
note that $\Gamma(\pi,1) \subseteq  \cO \times \cA$, we have
\begin{align}\label{step2eq1}
&  \sum_{\tau_1 \in  \Gamma(\pi,1)} \left\|  \left (\B_1(a_1,o_1;\theta) - 
     \B_1(a_1,o_1;\thetahat)\right)\b_0 (\theta) \right\|_1 \nn\\
     \le   & \sum_{(a,o)\in  \cA\times \cO}
     \left\|  \left (\B_1(a,o;\theta) - 
     \B_1(a,o;\thetahat)\right)\b_0(\theta) \right\|_1. 
\end{align}

Plugging \eqref{step1eq2} and \eqref{step2eq1} into \eqref{step0} completes the proof.
\end{proof}

\subsection{A hammer for studying confidence sets}
In this subsection, we develop a martingale concentration result, which forms the basis of analyzing confidence sets. 

We start by giving the following basic fact about POMDP. The proof is just some basic algebraic calculation so we omit it here.
\begin{restatable}{fact}{FactTensor}
\label{fact:tensor}
In $\pomdp(\theta)$, suppose $s_{h-1}$ is sampled from 
$\mu_{h-1}$, fix $a_{h-1} \equiv \tilde{a}$, and $a_h \equiv a$.
Then the joint distribution of $(o_{h+1},o_{h},o_{h-1})$ is 
$$
\Pr(o_{h+1}=\cdot,o_{h}=\cdot,o_{h-1}=\cdot) = (\O_{h+1}\T_{h}(a))\otimes \O_h \otimes (\O_{h-1} \diag(\mu_{h-1})\T_{h-1}(\at)\trans).
$$
By slicing the tensor, we can further obtain 
$$
\left\{
\begin{aligned}
\Pr(o_{h-1}=\cdot) & = \O_{h-1} \mu_{h-1},\\
\Pr(o_{h}=\cdot,o_{h-1}=\cdot) &= \O_{h} \T_{h-1}(\at) \diag(\mu_{h-1})\O_{h-1}\trans,\\
\Pr(o_{h+1}=\cdot,o_{h}=o,o_{h-1}=\cdot)  & = \O_{h+1} \T_{h}(a) \diag(\O_{h}(o\mid\cdot)) \T_{h-1}(\at) \diag(\mu_{h-1})\O_{h-1}\trans.
\end{aligned}
\right.
$$
\end{restatable}
A simple implication of Fact \ref{fact:tensor} is that if we execute policy $\pi$ from step $1$ to step $h-2$, take $\at$ and $a$ at step $h-1$ and $h$ respectively, then the joint distribution of $(o_{h+1},o_{h},o_{h-1})$ is the same as above except for replacing $\mu_{h-1}$ with $\Pr^\pi_\theta (s_{h-1}=\cdot)$.

Suppose we are given a set of sequential data $\{(o_{h+1}^{(t)},o_{h}^{(t)},o_{h-1}^{(t)})\}_{t=1}^{N}$ 
generated from $\pomdp(\theta)$ in the following way: 
at time $t$, execute policy $\pi_t$ from step $1$ to step $h-2$, take action $\tilde{a}$ at step $h-1$, and action $a$ at step $h$ respectively, and observe $(o_{h+1}^{(t)},o_{h}^{(t)},o_{h-1}^{(t)})$. 
 Here, we allow the policy $\pi_t$ to be \emph{adversarial}, in the sense that $\pi_t$ can be chosen based on 
$\{(\pi_i,o_{h+1}^{(i)},o_{h}^{(i)},o_{h-1}^{(i)})\}_{i=1}^{t-1}$.
Define $\mu_{h-1}^{adv}=\frac{1}{N}\sum_{t=1}^{N} \Pr^{\pi_t}_\theta(s_{h-1}=\cdot)$. Based on Fact \ref{fact:tensor}, we define the following probability vector, matrices and tensor,
$$
\left\{
\begin{aligned}
P_{h-1} & = \O_{h-1} \mu_{h-1}^{adv},\\
P_{h,h-1} &= \O_{h} \T_{h-1}(\at) \diag(\mu_{h-1}^{adv})\O_{h-1}\trans,\\
P_{h+1,h,h-1} & = (\O_{h+1}\T_{h}(a))\otimes \O_h \otimes (\O_{h-1} \diag(\mu_{h-1}^{adv})\T_{h-1}(\at)\trans)\\
P_{h+1,o,h-1} & = \O_{h+1} \T_{h}(a) \diag(\O_{h}(o\mid\cdot)) \T_{h-1}(\at) \diag(\mu_{h-1}^{adv})\O_{h-1}\trans,\quad o\in\cO.
\end{aligned}
\right.
$$
Accordingly, we define their empirical estimates as below
$$
\left\{
\begin{aligned}
\Phat_{h-1} & = \frac{1}{N}\sum_{t=1}^{N}  \e_{o_{h-1}^{(t)}},\\
\Phat_{h,h-1} &= \frac{1}{N}\sum_{t=1}^{N} \e_{o_{h}^{(t)}} \otimes\e_{o_{h-1}^{(t)}},\\
\Phat_{h+1,h,h-1} &= \frac{1}{N}\sum_{t=1}^{N}\e_{o_{h+1}^{(t)}} \otimes\e_{o_{h}^{(t)}} \otimes\e_{o_{h-1}^{(t)}},\\
\Phat_{h+1,o,h-1} &= \frac{1}{N}\sum_{t=1}^{N}\e_{o_{h+1}^{(t)}}  \otimes\e_{o_{h-1}^{(t)}} \one(o_{h}^{(t)}=o),\quad o\in\cO.
\end{aligned}
\right.
$$

\begin{lemma}\label{lem:adversary-concentration}
	There exists an absolute constant $c_1$, s.t. the following concentration bound holds with probability at least $1-\delta$
		\begin{align*}
 \max\bigg\{&\|\Phat_{h+1,h,h-1}-P_{h+1,h,h-1}\|_F,\|\Phat_{h,h-1}-P_{h,h-1}\|_F,\\
 &\max_{o\in\cO}\|\Phat_{h+1,o,h-1}-P_{h+1,o,h-1}\|_F,\|\Phat_{h-1}-P_{h-1}\|_2 \bigg\}
 \le c_1\sqrt{\frac{\log(ON/\delta)}{N}}.
\end{align*}
\end{lemma}
\begin{proof}
We start with proving that with probability at least $1-\delta/2$,
$$
\|\Phat_{h+1,h,h-1}-P_{h+1,h,h-1}\|_F \le c_1\sqrt{\frac{\log(ON/\delta)}{N}}.
$$
Let $\mathcal{F}_{t}$ be the $\sigma$-algebra generated by $\left\{\{\pi_{i}\}_{i=1}^{t+1},\{(o_{h+1}^{(i)},o_{h}^{(i)},o_{h-1}^{(i)})\}_{i=1}^{t}\right\}$.
$(\mathcal{F}_t)$ is a filtration. 
Define
$$
X_t =  e_{o_{h+1}^{(t)}} \otimes e_{o_{h}^{(t)}} \otimes e_{o_{h-1}^{(t)}} -
    (\O_{h+1}\T_{h}(a))\otimes \O_h \otimes (\O_{h-1} \diag(\Pr^{\pi_t}_\theta(s_{h-1}=\cdot))\T_{h-1}(\at)\trans).
$$
We have $X_t\in \mathcal{F}_t$ and $\E [X_t\mid \mathcal{F}_{t-1}]=\E [X_t\mid \pi_t]=0$, where the second identity follows from Fact \ref{fact:tensor}.
Moreover, 
\begin{align}\label{eq:oct16-1}
&	\|X_t\|_F \le \|X_t\|_1 \le \|  e_{o_{h+1}^{(t)}} \otimes e_{o_{h}^{(t)}} \otimes e_{o_{h-1}^{(t)}}\|_1 +  \nonumber \\ 
&\|(\O_{h+1}\T_{h}(a))\otimes \O_h \otimes (\O_{h-1} \diag(\Pr^{\pi_t}_\theta(s_{h-1}=\cdot))\T_{h-1}(\at)\trans)\|_1 = 2,
\end{align}
where $\|\cdot\|_1$ denotes the entry-wise $\ell_1$-norm of the tensor. 

Now, we can bound $\|\Phat_{h+1,h,h-1}-P_{h+1,h,h-1}\|_F$ by writing $\Phat_{h+1,h,h-1}-P_{h+1,h,h-1}$ as the sum of a sequence of tensor-valued martingale difference, vectorizing the tensors, and applying the standard vector-valued martingale concentration inequality   (e.g. see Corollary 7 in \cite{jin2019short}):
\begin{align*}
     &\|\Phat_{h+1,h,h-1}-P_{h+1,h,h-1}\|_F\\
     =&
      \|\frac{1}{N}\sum_{t=1}^{N} 
      \big(
      e_{o_{h+1}^{(t)}} \otimes e_{o_{h}^{(t)}} \otimes e_{o_{h-1}^{(t)}} -\\
&   (\O_{h+1}\T_{h}(a))\otimes \O_h \otimes (\O_{h-1} \diag(\Pr^{\pi_t}_\theta(s_{h-1}=\cdot))\T_{h-1}(\at)\trans)
      \big)
      \|_F\\
    =&  \|\frac{1}{N}\sum_{t=1}^{N} X_t \|_F \le  \bigO\left( \sqrt{\frac{\log(ON/\delta)}{N}}\right),
\end{align*}
with probability at least $1-\delta/2$. 
We remark that when vectoring a tensor, its Frobenius norm will become the $\ell_2$-norm the vector. 
So the upper bound of the norm of the vectorized martingales directly follows from \eqref{eq:oct16-1}.

Similarly, we can show that with probability at least $1-\delta/2$,
$$
\|\Phat_{h,h-1}-P_{h,h-1}\|_F \le \bigO\left( \sqrt{\frac{\log(ON/\delta)}{N}}\right)\quad \mbox{and} \quad 
\|\Phat_{h-1}-P_{h-1}\|_F \le \bigO\left( \sqrt{\frac{\log(ON/\delta)}{N}}\right).
$$
Using the fact
$\|\Phat_{h+1,o,h-1}-P_{h+1,o,h-1}\|_F \le \|\Phat_{h+1,h,h-1}-P_{h+1,h,h-1}\|_F$ completes the whole proof.
\end{proof}

\subsection{Properties of confidence sets}

For convenience of discussion, we divide the constraints in $\Theta_k$ into three categories as following

{\bf Type-0 constraint: }
$$\norm{k \cdot \b_0(\hat{\theta}) - \bp^k}_2 \le  \beta_k\}$$

{\bf Type-I constraint: }
$$\norm{\B_1(a, o; \hat{\theta}) \bP_{1}^k (a, \tilde{a}) -\bQ_1^k (o, a, \tilde{a})}_F \le \gamma_k,$$
where $\bQ_1^k$ and $ \bP_{1}^k$ are actually equivalent to $O$-dimensional counting vectors because there is no observation (or only a dummy observation) at step $0$, which implies each of them has only one non-zero column.
    With slight abuse of notation, we use $\bQ_1^k$ and $ \bP_{1}^k$ to denote their non-zero columns in the following proof.

{\bf Type-II constraint: } for $2\le h\le H-1$,
$$\norm{\B_h(a, o; \hat{\theta}) \bP_{h}^k (a, \tilde{a}) -\bQ_h^k (o, a, \tilde{a})}_F \le \gamma_k$$

Recalling the definition of $\bp^k(\theta)$, $\bP_{h}^k (a, \tilde{a})$ and  $\bQ_h^k (o, a, \tilde{a})$ and applying Lemma \ref{lem:adversary-concentration}, we get the following concentration  results.
\begin{corollary}\label{cor:confidence-concentrate}
Let $\theta^\star = (\T,\O,\mu_1)$.	By applying Lemma \ref{lem:adversary-concentration} directly, with probability at least $1-{\delta}$, for all $k\in[K]$ and $(o,a,\at)\in\cO \times \cA^2$, we have	
$$
	\left\{
\begin{aligned}
		&\left\|  \frac{1}{k}\bp^k - \O_{1}\mu_1 \right\|_2 \le 
		\bigO\left(\sqrt{\frac{\iota}{k}}\right),\\
		&\left\|  \frac{1}{k}\bP_{1}^k (a, \tilde{a})  - \O_{1}\mu_1 \right\|_2 \le 
		\bigO\left(\sqrt{\frac{\iota}{k}}\right),\\
			&\left\|  \frac{1}{k}\bQ_1^k (o, a, \tilde{a}) - \left(\O_{2} \T_{1}(\at) \diag(\mu_{1})\O_{1}\trans\right)_o \right\|_2 \le 
		\bigO\left(\sqrt{\frac{\iota}{k}}\right),\\
			&\left\|  \frac{1}{k}\bP_h^k(a,\at) - \underbrace{\O_{h} \T_{h-1}(\at) \diag(\mu_{h-1}^k)\O_{h-1}\trans}_{\textstyle \V} \right\|_F \le 
		\bigO\left(\sqrt{\frac{\iota}{k}}\right),\\
		&\left\|  \frac{1}{k}\bQ_h^k(o,a,\at) - 
		\underbrace{\O_{h+1} \T_{h}(a) \diag(\O_{h}(o\mid\cdot)) \T_{h-1}(\at) \diag(\mu_{h-1}^k)\O_{h-1}\trans}_{\textstyle \W} \right\|_F \le 
		\bigO\left(\sqrt{\frac{\iota}{k}}\right),
\end{aligned}
\right.$$
where 
$$
\iota = \log(KAOH/\delta)
\quad \mbox{and} \quad 
\mu_{h-1}^k =\frac{1}{k} \sum_{t=1}^{k} \Pr^{\pi_t}_{\theta^\star}(s_{h-1}=\cdot)\quad 2\le h\le H-1.
$$
Note that for all $k\in[K]$, $\mu_{1}^k = \mu_1$ independent of $\pi_1,\ldots,\pi_k$.
\end{corollary}

Now, with Corollary \ref{cor:confidence-concentrate}, we can prove the true parameter $\theta^\star$ always lies in the confidence sets for $k\in[K]$ with high probability.
 
\begin{lemma}\label{prop:realizability}
Denote by $\theta^\star = (\T,\O, \mu_1)$ the the ground truth parameters of the POMDP.
With probability at least $1-\delta$, we have 
$\theta^\star \in \Theta_k$ for all $k \in [K]$.
\end{lemma}
\begin{proof}
By the definition of $\b_0(\theta^\star)$ and $\B_h(a,o;\theta^\star)$, we have
$$	(*)\left\{
\begin{aligned}
 & \b_0(\theta^\star) = \O_1 \mu_1,\\
&   \left(\O_{2} \T_{1}(\at) \diag(\mu_{1})\O_{1}\trans\right)_o
 = \B_1(\at,o;\theta^\star)  \O_{1}\mu_1,\\
 &		\W = \B_h(a,o;\theta^\star) \cdot \V, \quad h\ge 2,
		\end{aligned}
\right.$$
where $\W$ and $\V$ are shorthands defined in Corollary 
\ref{cor:confidence-concentrate}.

It's easy to see $(*)$ and Corollary\ref{cor:confidence-concentrate} directly imply $\left\| \bp^k - \b_0(\theta^\star) \right\|_2 \le 
		\bigO\left(\sqrt{{k\iota}}\right)$ and thus $\theta^\star$ satisfies Type-0 constraint. For other constraints, we have

{\bf Type-I constraint: }
\begin{align*}
&\norm{\bQ_1^k (o, a, \tilde{a})  - \B_1(\at,o;\theta^\star)\bP_{1}^k (a, \tilde{a}) }_2\\
\le &\norm{ \bQ_1^k (o, a, \tilde{a})  - k\left(\O_{2} \T_{1}(\at) \diag(\mu_{1})\O_{1}\trans\right)_o }_2			+ \norm{ \B_1(\at,o;\theta^\star)(k\O_1\mu_1-\bP_{1}^k (a, \tilde{a}))}_2
\\
&+  k\norm{\left(\O_{2} \T_{1}(\at) \diag(\mu_{1})\O_{1}\trans\right)_o - \B_1(\at,o;\theta^\star)\O_1\mu_1 }_2 \\
= & \norm{ \bQ_1^k (o, a, \tilde{a})  - k\left(\O_{2} \T_{1}(\at) \diag(\mu_{1})\O_{1}\trans\right)_o }_2			+ \norm{ \B_1(\at,o;\theta^\star)(k\O_1\mu_1-\bP_{1}^k (a, \tilde{a}))}_2\\
\le & \norm{ \bQ_1^k (o, a, \tilde{a})  - k\left(\O_{2} \T_{1}(\at) \diag(\mu_{1})\O_{1}\trans\right)_o }_2			+ \norm{ \B_1(\at,o;\theta^\star)}_2\norm{k\O_1\mu_1-\bP_{1}^k (a, \tilde{a})}_2\\
\le & \bigO\left({\frac{\sqrt{kS\iota}}{\alpha}}\right)
\end{align*}
 where the identity follows from $(*)$, and the last inequality follows from Corollary\ref{cor:confidence-concentrate} and 
 \begin{align*}
 	 \norm{\B_h(a, o; \theta^\star)}_2 &= \norm{ \O_{h+1} \T_h(a) \diag(\O_h(o|\cdot)) \O_h^{\dagger}}_2\\
& \le \frac{1}{\alpha}\norm{ \O_{h+1} \T_h(a) \diag(\O_h(o|\cdot))}_2\\
& \le \frac{\sqrt{S}}{\alpha}\norm{ \O_{h+1} \T_h(a) \diag(\O_h(o|\cdot))}_{1\rightarrow 1} \le \frac{\sqrt{S}}{\alpha}.
 \end{align*}
 
{\bf Type-II constraint: } similarly, for $h\ge 2$, we have
\begin{align*}
&  \norm{\B_h(a, o; \theta^\star) \bP_{h}^k (a, \tilde{a}) -\bQ_h^k (o, a, \tilde{a})}_F\\
\le  &  k\norm{ \B_h(a, o; \theta^\star) \cdot\V - \W}_F
+   \norm{\B_h(a, o; \theta^\star) (\bP_{h}^k (a, \tilde{a}) -k \V)}_F
+ \norm{ k\W-\bQ_h^k (o, a, \tilde{a})}_F\\
= & \norm{\B_h(a, o; \theta^\star) (\bP_{h}^k (a, \tilde{a}) -k \V)}_F + \norm{ k\W-\bQ_h^k (o, a, \tilde{a})}_F\\
\le & \norm{\B_h(a, o; \theta^\star) }_2\norm{\bP_{h}^k (a, \tilde{a}) -k \V}_F
+ \norm{ k\W-\bQ_h^k (o, a, \tilde{a})}_F\\
\le & \bigO\left({\frac{\sqrt{kS\iota}}{\alpha}}\right),
 \end{align*}
 
 Therefore, we conclude that 
 $\theta^\star \in \Theta_k$ for all $k \in [K]$ with probability at least $1-\delta$.
\end{proof}

Furthermore, with Corollary \ref{cor:confidence-concentrate}, we can prove the following bound for operator error.
\begin{lemma}\label{prop:confidence-set-ucb}
With probability at least $1-{\delta}$, for all $k\in[K]$, $\thetahat=(\Ohat,\That,\hat{\mu}_1)\in\Theta_{k+1}$ and $(o,a,\at,h)\in\cO \times \cA^2\times \{2,\ldots,H-1\}$, we have	
$$	\left\{
\begin{aligned}
&	\left\| \b_0(\theta^\star)-\b_0(\thetahat)  \right\|_2 
 \le \bigO\left(\sqrt{\frac{\iota}{k}}\right),\\
 & \left\|  \left (\B_1(\at,o;\thetahat) - 
     \B_1(\at,o;\theta^\star)\right)\b_0(\theta^\star)\right\|_2 \le \bigO\left(\sqrt{\frac{S\iota}{k\alpha^2}}\right),\\
&  \sum_{s=1}^{S} \left\| \left(\B_h(a,o;\thetahat) -\B_h(a,o;\theta^\star)\right)  (\O_{h} \T_{h-1}(\at))_s \right\|_1
    \sum_{t=1}^{k} \Pr^{\pi_t}_{\theta^\star} (s_{h-1}=s)
     \le \bigO\left(\sqrt{\frac{k S^2 O\iota}{\alpha^4}}\right),
     		\end{aligned}
\right.$$
where $\iota = \log(KAOH/\delta)$.
\end{lemma}
\begin{proof}
For readability, we copy the following set of identities from Lemma \ref{prop:realizability} here,
$$	(*)\left\{
\begin{aligned}
 & \b_0(\theta^\star) = \O_1 \mu_1,\\
&   \left(\O_{2} \T_{1}(\at) \diag(\mu_{1})\O_{1}\trans\right)_o
 = \B_1(\at,o;\theta^\star)  \O_{1}\mu_1,\\
 &		\W = \B_h(a,o;\theta^\star) \cdot \V, \quad h\ge 2.
		\end{aligned}
\right.$$

{\bf Type-0 closeness: }
\begin{align*}
 \left\| \b_0(\theta^\star)-\b_0(\thetahat)  \right\|_2 
 \le  \left\| \frac{1}{k}\bp^k- \b_0(\theta^\star) \right\|_2
 + \left\| \frac{1}{k}\bp^k- \b_0(\thetahat) \right\|_2 
 \le \bigO\left(\sqrt{\frac{\iota}{k}}\right),
\end{align*}
where the last inequality follows from $(*)$, Corollary\ref{cor:confidence-concentrate} and $\hat{\theta}\in \Theta_{k+1}$.

{\bf Type-I closeness: } similarly, we have
\begin{align*}
	&\left\|  \left (\B_1(\at,o;\thetahat) - 
     \B_1(\at,o;\theta^\star)\right)\b_0(\theta^\star)\right\|_2 \\
     \le & \left\|   \left(\O_{2} \T_{1}(\at) \diag(\mu_{1})\O_{1}\trans\right)_o- 
     \B_1(\at,o;\theta^\star)\b_0(\theta^\star)\right
       \|_2 \\
&      + \norm{ \left(\O_{2} \T_{1}(\at) \diag(\mu_{1})\O_{1}\trans\right)_o- \B_1(\at,o;\thetahat)\b_0(\theta^\star) }_2\\
     = & \norm{  \left(\O_{2} \T_{1}(\at) \diag(\mu_{1})\O_{1}\trans\right)_o- \B_1(\at,o;\thetahat)\b_0(\theta^\star) }_2\\
     \le & \norm{ \left(\O_{2} \T_{1}(\at) \diag(\mu_{1})\O_{1}\trans\right)_o- \frac{1}{k}\bQ_1^k (o, a, \tilde{a})  }_2 +
\frac{1}{k}     \norm{ \bQ_1^k (o, a, \tilde{a})  - \B_1(\at,o;\thetahat)\bP_{1}^k (a, \tilde{a}) }_2\\
     & + \norm{\B_1(\at,o;\thetahat)\left(\frac{1}{k}\bP_{1}^k (a, \tilde{a}) - \b_0(\theta^\star)\right)}_2\\
      \le & \norm{ \left(\O_{2} \T_{1}(\at) \diag(\mu_{1})\O_{1}\trans\right)_o- \frac{1}{k}\bQ_1^k(o,a,\at) }_2 +
\frac{1}{k}     \norm{ \bQ_1^k (o, a, \tilde{a})  - \B_1(\at,o;\thetahat)\bP_{1}^k (a, \tilde{a}) }_2\\
     & + \norm{\B_1(\at,o;\thetahat)}_2\norm{\frac{1}{k}\bP_{1}^k (a, \tilde{a}) - \O_1 \mu_1}_2\\
     \le & \bigO\left(\sqrt{\frac{S\iota}{k\alpha^2}}\right),
\end{align*}
where the identity follows from $(*)$ and the last inequality follows from Corollary\ref{cor:confidence-concentrate} and $\hat{\theta}\in \Theta_{k+1}$.

{\bf Type-II closeness: } we continue to use the same proof strategy, for $h\ge 2$ 
\begin{align}\label{eq:Jun3-1}
\nn	&\left\|  \left (\B_h(a,o;\thetahat) - 
     \B_h(a,o;\theta^\star)\right) \V\right\|_F \\
     \le & \left\|  \W- 
     \B_h(a,o;\theta^\star) \V\right
       \|_F
    \nn   + \norm{\frac{1}{k}\bQ_h^k(o,a,\at)-\W}_F
       \\
       &+\frac{1}{k} \norm{ \B_h(a,o;\thetahat) \bP_h^k(a,\at) -\bQ_h^k(o,a,\at)}_F+\norm{ \B_h(a,o;\thetahat) \left(\V -\frac{1}{k} \bP_h^k(a,\at)\right)}_F \nn\\
     = &  \norm{\frac{1}{k}\bQ_h^k(o,a,\at)-\W}_F+ \frac{1}{k}\norm{ \B_h(a,o;\thetahat) \bP_h^k(a,\at) -\bQ_h^k(o,a,\at)}_F       \nn\\
       &+\norm{ \B_h(a,o;\thetahat) \left(\V - \frac{1}{k}\bP_h^k(a,\at)\right)}_F \nn \\
     \le &  \bigO\left(\sqrt{\frac{S\iota}{k\alpha^2}}\right),
  \end{align}
      where the identity follows from $(*)$ and the
      last inequality follows from Corollary\ref{cor:confidence-concentrate} and $\hat{\theta}\in \Theta_{k+1}$.
      
      Recall $\V = \O_{h} \T_{h-1}(\at) \diag(\mu_{h-1}^k)\O_{h-1}\trans$ and utilize Assumption \ref{assump:POMDP},
      \begin{align*}
	&\left\|  \left (\B_h(a,o;\thetahat) - 
     \B_h(a,o;\theta^\star)\right) \V\right\|_F \\
     \ge & {\alpha}
     \left\|  \left (\B_h(a,o;\thetahat) - 
     \B_h(a,o;\theta^\star)\right) \O_{h} \T_{h-1}(\at) \diag(\mu_{h-1}^k)\right\|_F \\
     \ge & \frac{\alpha}{\sqrt{SO}}
     \left\|  \left (\B_h(a,o;\thetahat) - 
     \B_h(a,o;\theta^\star)\right) \O_{h} \T_{h-1}(\at) \diag(\mu_{h-1}^k)\right\|_1\\
     = & \frac{\alpha}{k\sqrt{SO}} \sum_{s=1}^{S} \left\| \left(\B_h(a,o;\thetahat) -\B_h(a,o;\theta^\star)\right)  (\O_{h} \T_{h-1}(\at))_s \right\|_1
    \sum_{t=1}^{k} \Pr^{\pi_t}_{\theta^\star}(s_{h-1}=s).
     \end{align*}
   Plugging it back into \eqref{eq:Jun3-1} completes the whole proof.
\end{proof}

\subsection{Proof of Theorem \ref{thm:main}}
\label{app:proof-mainthm}

In order to utilize Lemma \ref{prop:confidence-set-ucb} to bound the operator error in Lemma \ref{prop:subopt-bound}, we need the following algebraic transformation.
Its proof is postponed to Appendix \ref{appendix:auxiliary}.
\begin{lemma}\label{lem:root-regret}
Let $z_k\in[0,C_z]$  and $w_k\in[0,C_w]$ for $k\in\mathbb{N}$.
Define $S_{k} = \sum_{j=1}^{k} w_{j}$ and $S_{0}=0$.
If $  z_k   S_{k-1}\le C_z C_w C_0 \sqrt{k}$ for any $k\in [K]$, we have
$$
\sum_{k=1}^{K}   z_k w_k \le  2C_z C_w (C_0+1) \sqrt{K}\log(K).
$$
Moreover, there exists some hard case where we have a almost matching lower bound  $O\left(C_z C_w C_0 \sqrt{K}\right)$. 	
\end{lemma}

Now, we are ready to prove the main theorem based on Lemma \ref{prop:subopt-bound}, Lemma \ref{prop:confidence-set-ucb} and Lemma \ref{lem:root-regret}.
\UCBTheorem*
\begin{proof}

There always exist an optimal deterministic policy $\pi^\star$ for the ground truth $\pomdp(\theta^\star)$, i.e., $V^\star(\theta^\star) = V^{\pi^\star}(\theta^\star)$.
WLOG, we can always choose the greedy policy $\pi_k$ to be deterministic, i.e., the probability to take any action given a history is either $0$ or $1$. 

By Lemma \ref{prop:realizability}, we have $\theta^\star \in \Theta_{k}$ for all $k\in[K]$ with probability at least $1-\delta$.
 Recall that $(\pi_k,\theta_k)=\arg\max_{\pi,\theta\in\Theta_k} V^\pi(\theta)$, so with probability at least $1-\delta$, we have
\begin{align}\label{eq:apr30-step1-1}
    &\sum_{k=1}^{K} \left(V^{\pi^\star}(\theta^\star)  - V^{\pi_k}(\theta^\star)  \right) \nn\\
    \le & \sum_{k=1}^{K} \left( V^{\pi_k}(\theta_k) - V^{\pi_k}(\theta^\star) \right) \nn \\
    \le & H \sum_{k=1}^{K} \sum_{[o_{H},\tau_{H-1}]\in \cO \times \Gamma(\pi_k,{H-1})} \| \Pr^{\pi_k}_{\theta^\star} ([o_{H},\tau_{H-1}]) -\Pr^{\pi_k}_{\theta_k} ([o_{H},\tau_{H-1}])\|_1 \nn \\
    = & H \sum_{k=1}^{K} \sum_{\tau_{H-1}\in \Gamma(\pi_k,{H-1})} \| \b(\tau_{H-1};\theta^\star) -\b(\tau_{H-1};\theta_k)\|_1,
\end{align}
where the identity follows from Fact \ref{fact:belief}.

Applying Lemma \ref{prop:subopt-bound}, we have 
\begin{align}\label{eq:apr30-step1-2}
&     \sum_{\tau_{H-1}\in \Gamma(\pi_k,H-1)} \| \b(\tau_{H-1};\theta^\star) -\b(\tau_{H-1};\theta_k)\|_1 \nn\\
   \le & \underbrace{\frac{\sqrt{S} }{\alpha} \left\| \b_0(\theta^\star)- \b_0(\theta_k) \right\|_1  
    + \frac{\sqrt{S} }{\alpha}  \sum_{(a,o)\in  \cA\times \cO}
     \left\|  \left (\B_1(a,o;\theta_k) - 
     \B_1(a,o;\theta^\star)\right)\b_0(\theta^\star)\right\|_1}_{J_k}\nn \\
 +   \frac{\sqrt{S} }{\alpha} &
     \sum_{h=2}^{H-1}\sum_{(a,\at,o)\in\cA^2\times\cO}
     \sum_{s=1}^{S} \left\| \left(\B_h(a,o;\theta_k) -\B_h(a,o;\theta^\star)\right)  (\O_{h} \T_{h-1}(\at))_s \right\|_1
     \Pr^{\pi_k}_{\theta^\star}(s_{h-1}=s).
\end{align}
We can bound the first two terms by Lemma \ref{prop:confidence-set-ucb}, and obtain that with probability at least $1-\delta$,
\begin{align}\label{eq:apr30-step1-3}
H\sum_{k=1}^{K} J_k\le \bigO\left(\frac{HSAO}{\alpha^2}\sqrt{K\iota}\right).
\end{align}

Plugging \eqref{eq:apr30-step1-2} and \eqref{eq:apr30-step1-3} into \eqref{eq:apr30-step1-1}, we obtain
\begin{align}\label{eq:apr30-step1-4}
&     \sum_{k=1}^{K} \left(V^{\pi^\star}(\theta^\star)  - V^{\pi_k}(\theta^\star)  \right)  
    \le  \bigO\left(\frac{HSAO}{\alpha^2}\sqrt{K\iota}\right) +\nn \\
     &   \frac{H^2S^{1.5}A^2O}{\alpha} \max_{s,o,a,\at,h}
   \sum_{k=1}^{K}  \left\| \left(\B_h(a,o;\theta_k) -\B_h(a,o;\theta^\star)\right)  (\O_{h} \T_{h-1}(\at))_s \right\|_1
     \Pr^{\pi_k}_{\theta^\star}(s_{h-1}=s).
     \end{align}
It remains to bound the second term.

By Lemma \ref{prop:confidence-set-ucb}, with probability at least $1-{\delta}$, for all $k\in[K]$, $\theta_k \in\Theta_{k}$ and $(s,o,a,\at,h)\in\cS \times \cO \times \cA^2\times \{ 2,\ldots,H-1\}$, we have
\begin{align}\label{eq:apr30-step2-1}
   \underbrace{\left\| \left(\B_h(a,o;\theta_k) -\B_h(a,o;\theta^\star)\right)  (\O_{h} \T_{h-1}(\at))_s \right\|_1 }_{\textstyle z_k}
    \sum_{t=1}^{k-1} 
    \underbrace{\Pr^{\pi_t}_{\theta^\star}(s_{h-1}=s)}_{\textstyle w_t} 
     \le \bigO\left(\sqrt{\frac{k S^2 O\iota}{\alpha^4}}\right).
\end{align}

By simple calculation, we have $z_k \le \sqrt{S}/\alpha$.
Invoking Lemma \ref{lem:root-regret} with \eqref{eq:apr30-step2-1}, we obtain 
\begin{align}\label{eq:apr30-step2-2}
   \sum_{k=1}^{K} w_k z_k   \le \bigO\left(\frac{\sqrt{S^3O\iota}}{\alpha^3}\sqrt{K}\log(K)\right).
\end{align}

Plugging \eqref{eq:apr30-step2-2} back into \eqref{eq:apr30-step1-4} gives
\begin{align}
   \sum_{k=1}^{K} \left(V^{\pi^\star}(\theta^\star)  - V^{\pi_k}(\theta^\star)  \right)  
    \le  \bigO\left(\frac{H^2S^3A^2O^{1.5}\sqrt{\iota}}{\alpha^4}\sqrt{K}\log(K)\right).
     \end{align}

Finally, choosing  
$$
K_{\max} = \bigO\left(\frac{H^4S^6A^4O^3\log(HSAO/\eps)}{\alpha^8\eps^2}\right),
$$
 and outputting a policy from $\{ \pi_1,\ldots,\pi_K\}$ uniformly at random complete the proof.
\end{proof}

%\newpage

\section{Learning POMDPs with Deterministic Transition}
\label{app:determinisitc}

In this section, we introduce a computationally and statistically efficient algorithm for POMDPs with deterministic transition. A sketched proof is provided.

We comment that  some  previous works have studied POMDPs with deterministic transitions or deterministic emission process assuming the model is \emph{known} (e.g. \cite{bazinin2018iterative,besse2009quasi,bonet2012deterministic}); their results mainly focus on the planning aspect. In contrast, we assume \emph{unknown} models which requires to learn the transition and emission process first. 
It is also worth mentioning that the (quasi)-deterministic POMDPs defined in these works are not exactly the same as ours. For example, the deterministic POMDPs in \cite{bonet2012deterministic} refer to those with stochastic initial state but deterministic emission process, while we assume deterministic initial state but stochastic emission process. 
Therefore, their computational hardness results do not conflict with the efficient algorithm in this section.

\begin{algorithm}[H]
 \caption{Learning POMDPs with Deterministic Transition}
 \label{alg:deterministic}
 \begin{algorithmic}[1]
 \STATE {\bf initialize} $N=C \log(HSA/p)/(\min\{\epsilon/{(\sqrt{O}H)},\xi\})^2$, $n_{h}=\one(h=1) $ 
 for all $h\in[H]$.
\FOR{$h=1,\ldots,H-1$}
 \FOR{$(s,a)\in [n_{h}] \times \cA$}
  \STATE Reset $z \leftarrow \mathbf{0}_{O\times 1}$ and 
   $t \leftarrow n_{h+1} +1$
 \FOR{$i \in [N]$}
 \STATE execute policy $\pi_{h}(s)$ from step $1$ to step $h-1$, take action $a$ at $h^{\rm th}$ step and observe $o_{h+1}$
 \vspace{-3.5mm}
 \STATE $z \leftarrow z + \frac{1}{N} e_{o_{h+1}}$ 
  \ENDFOR
\FOR{$s\pr\in[n_{h+1}]$}
\IF{$\| \phi_{h+1,s\pr} - z \|_2 \le 0.5{\xi}$}
\STATE {$t \leftarrow {s\pr}$}
\ENDIF
\ENDFOR
\IF{$t=n_{h+1} + 1$}
\STATE $n_{h+1} \leftarrow  n_{h+1} +1 $ 
\STATE $\phi_{h+1,n_{h+1}} \leftarrow z$ and $\pi_{h+1}(n_{h+1}) \leftarrow a \circ \pi_{h}(s)$
\ENDIF
\STATE Set the $s^{\rm th}$ column of $\That_{h,a}$  to be $e_t$
\ENDFOR
\ENDFOR
\STATE {\bf output} $\hat{\mu}_0 = e_1$ and  $\left\{ n_h, \  \{\That_{h,a}\}_{a\in\cA}\ \mbox{and} \ \{\phi_{h,i}\}_{i\in [n_h]}:\  h\in[H]\right\}$
 \end{algorithmic}
\end{algorithm}

\DetTheorem*

\begin{proof}
	The algorithm works  by inductively finding all the states we can reach at each step, utilizing the property of deterministic transition and good separation between different observation vectors. We sketch a proof based on induction below.
	
	We say a state $s$ is $h$-step reachable if there exists a policy $\pi$ s.t. $\Pr^\pi(s_h=s)=1$.
	In our algorithm, we use $n_h$ to denote the number of $h$-step reachable states. 
	All the policies mentioned below is a sequence of fixed actions (independent of observations).
	
	%\chijin{explain a bit what is $n_h$. why we are using $i$ here instead of $s$, and what is $f(\cdot)$?} 
	Suppose at step $h$, there are $n_h$ $h$-step reachable states and we can reach the $s^{\rm th}$ one of them at the $h^{\rm th}$ step by executing a \emph{known} policy $\pi_h(s)$. 
	Note that for every state $s'$ that is $(h+1)$-step reachable, there must exist some state $s$ and action $a$ s.t. $s$ is $h$-step reachable and $\T_h(s'\mid s,a)=1$.
		Therefore, based on our induction assumption, we can reach all the $(h+1)$-step reachable states by executing all $a\circ\pi_h(s)$ for $(a,s)\in\cA\times[n_h]$. 
		
Now the problem is how to tell if we reach the same state by executing two different $a\circ\pi_h(s)$'s. The solution is to look at the distribution of $o_{h+1}$. Because the POMDP has deterministic transition, we always reach the same state when executing the same $a\circ\pi_h(s)$ and hence the distribution of $o_{h+1}$ is exactly the distribution of observation corresponding to that state. 
 By Hoeffding's inequality, for each fixed $a\circ\pi_h(s)$, we can estimate the distribution of $o_{h+1}$ with $\ell_2$-error smaller than $\xi/8$ with high probability using $N \ge \tilde{\Omega}(1/\xi^2)$ samples. 
    Since the observation distributions of two different states have $\ell_2$-separation no smaller than $\xi$, we can tell if two different $a\circ\pi_h(s)$'s reach the same state by looking at the distance between their distributions of $o_{h+1}$. For those policies reaching the same state, we only need to keep one of them, so there are at most $S$ policies kept ($n_{h+1} \le S$).	
		
		By repeating the argument inductively from $h=1$ to $h=H$, we can recover the exact transition dynamics  $\T_h(\cdot\mid s,a)$ and get an high-accuary estimate of $\O_h(\cdot\mid s)$ for every $h$-step reachable state $s$ and all $(h,a)\in [H]\times \cA$ \emph{up to permutation of states}.
		Since the POMDP has deterministic transition, we can easily find the optimal policy of the estimated model by dynamic programming. 

		The $\epsilon$-optimality simply follows from the fact that when
		$N \ge \tilde{\Omega}(H^2O/\epsilon^2)$, we have the estimated distribution of observation for each state being $\bigO(\epsilon/H)$ accurate in $\ell_1$-distance for all reachable states. This implies that the optimal policy of the estimated model is at most $\bigO(\epsilon/H) \times H = \bigO(\epsilon)$ suboptimal.
		The overall sample complexity follows from our requirement $N \ge \max\{\tilde{\Omega}(H^2O/\epsilon^2), \tilde{\Omega}(1/\xi^2)\}$, and the fact we need to run $N$ episodes for each $h\in [H], s\in\cS, a\in\cA$.
		%\chijin{I also think the overall sample and computation is one $H$ factor more, as executing policies may require $H$ samples.}
\end{proof}

%\newpage
\section{Auxiliary Results}
\label{appendix:auxiliary}

\subsection{Derivation of equation \eqref{eq:OOM_pb}}
\label{app:pomdp-oom}
When conditioning on a fixed action sequence $\{a_{H-1},\ldots,a_1\}$, a POMDP will reduce to a non-stationary HMM, whose transition matrix and observation matrix at $h^{\rm th}$ step are $\T_{h}(a_h)$ and $\O_{h}$, respectively. 
So $\Pr(o_H, \ldots, o_1 | a_{H-1}, \ldots, a_1)$ is equal to the probability of observing $\{o_{H},\ldots,o_1\}$ in this particular HMM.
Using the basic properties of HMMs, we can easily express $\Pr(o_H, \ldots, o_1 | a_{H-1}, \ldots, a_1)$ in terms of the transition and observation matrices
$$
\O_{H}(o_{H}|\cdot) \cdot  [\T_{H-1}(a_{H-1}) \diag(\O_{H-1}(o_{H-1}|\cdot))]  \cdots [\T_{1}(a_{1}) \diag(\O_{1}(o_1|\cdot))] \cdot \mu_1.
$$
Recall the definition of operators 
$$
\B_h(a, o) =  \O_{h+1} \T_h(a) \diag(\O_h(o|\cdot)) \O_h^{\dagger}, \qquad \b_0 = \O_1  \mu_1,
$$
and $\O_h^\dagger\O_h = \I_S$, we conclude that  
$$
\Pr(o_H, \ldots, o_1 | a_{H-1}, \ldots, a_1) = \e_{o_H}\trans \cdot \B_{H-1}(a_{H-1}, o_{H-1}) \cdots \B_1(a_1, o_1) \cdot \b_0.
$$

\subsection{Derivation of equation \eqref{eq:justification2}}
\label{app:justification2}
 Note that $\pi$ is a deterministic policy and $\Gamma(\pi,H)$ is a set of all the observation and action sequences of length $H$ that could occur under policy $\pi$, i.e., for any $\tau_H=(o_H,\ldots,a_1,o_1)\in \Gamma(\pi,H)$, we have  
$\pi(a_{H-1} \ldots, a_1 \mid o_H, \ldots, o_1) = 1$, and $\pi(a'_{H-1} \ldots, a'_1 \mid o_H, \ldots, o_1) = 0$ for any action sequence $(a'_{H-1} \ldots, a'_1) \neq (a_{H-1} \ldots, a_1)$. Therefore, for $\tau_H\in \Gamma(\pi,H)$, we have:
\begin{align*}
\Pr^{\pi}_\theta(o_H, \ldots, o_1) = &\sum_{a'_{H-1}\in \cA} \cdots\sum_{a'_{1}\in \cA}\Pr^{\pi}_\theta(o_H, a'_{H-1}, \ldots, a'_1,  o_1)\\
= & \Pr^{\pi}_\theta(o_H, a_{H-1}, \ldots, a_1,  o_1)\\
 =&\left[\prod_{h=1}^{H-1}\pi(a_h \mid o_h,a_{h-1},\ldots,a_1,o_1)\right]\cdot \left[\prod_{h=1}^{H}\Pr_\theta(o_h \mid a_{h-1},o_{h-1}\ldots,a_1,o_1)\right]\\
 =& \prod_{h=1}^{H}\Pr_\theta(o_h \mid a_{h-1},o_{h-1},\ldots,a_1,o_1)\\
 =&\Pr_{\theta}(o_H, \ldots, o_1|a_{H-1}, \ldots, a_1).
\end{align*}

\subsection{Boosting the success probability}
\label{app:justification3}
We can run Algorithm \ref{alg:UCB} independently for $n=\bigO(\log(1/\delta))$ times and obtain $n$ policies. 
Each policy is independent of others and is $\eps$-optimal with probability at least $2/3$. So with probability at least $1-\delta/2$, at least one of them will be $\eps$-optimal. 
In order to evaluate their performance, it suffices to run each policy for $\bigO(\log(n/\delta)H^2/\eps^2)$ episodes and use the empirical average of the cumulative reward as an estimate.
By standard concentration argument, with probability at least $1-\delta/2$, the estimation error for each policy is smaller than $\eps$. Therefore, if we pick the policy with the best empirical performance, then with probability at least $1-\delta$, it is $3\eps$-optimal. Rescaling $\eps$ gives the desired accuracy.
It is direct to see that the boosting procedure will only incur an additional $\rm{polylog}(1/\delta)$ factor in the final sample complexity, and thus will not hurt the optimal dependence on $\eps$.

\subsection{Basic facts about POMDPs and the operators}
In this section, we provide some useful facts about POMDPs. Since their proofs are quite straightforward, we omit them here.

The following fact gives two linear equations the operators always satisfy.  
Its proof simply follows  from the  definition of the operators and Fact \ref{fact:tensor}.
\begin{fact}\label{fact:equation}
In the same setting as Fact~\ref{fact:tensor}, suppose  Assumption \ref{assump:POMDP} holds, then we have
$$
\left\{
\begin{aligned}
\Pr(o_{h}=\cdot,o_{h-1}=\cdot) \e_o &= \B_h(\at,o;\theta) \Pr(o_{h-1}=\cdot) ,\\
\Pr(o_{h+1}=\cdot,o_{h}=o,o_{h-1}=\cdot)  &= 
\B_h(a,o;\theta) \Pr(o_{h}=\cdot,o_{h-1}=\cdot).
\end{aligned}
\right.
$$
\end{fact}

The following fact relates (unnormalized) belief states to distributions of observable sequences. 
Its proof follows from simple computation using conditional probability formula and $\O_h^\dagger\O_h = \I_S$.
\begin{fact}\label{fact:belief}
For any $\pomdp(\theta)$ satisfying Assumption \ref{assump:POMDP}, deterministic policy $\pi$ and  $[o_{h},\tau_{h-1}]\in \cO \times \Gamma(\pi,{h-1})$, we have
$$ \e_{o_h}\trans\b(\tau_{h-1};\theta)  =\Pr^{\pi}_{\theta} ([o_{h},\tau_{h-1}]),$$ 
where $\Pr^{\pi}_{\theta} ([o_{h},\tau_{h-1}])$ is the probability of observing $[o_{h},\tau_{h-1}]$ when executing policy $\pi$ in $\pomdp(\theta)$.
\end{fact}

%\subsection{Proof of Lemma \ref{lemma:belief-induction}}
%We prove for the case $X=\Ot^\dagger$ and the other case follows almost the same.

%\subsection{Proof of Lemma \ref{lem:regret-decomp}}

%\subsection{Proof of Lemma \ref{lem:belief-pairdistribution}}

%\subsection{Proof of Lemma \ref{prop:subopt-bound}}

%\subsection{Proof of Lemma \ref{lem:adversary-concentration}}

%\subsection{Proof of Lemma \ref{prop:realizability}}

%\subsection{Proof of Lemma \ref{prop:confidence-set-ucb}}

\subsection{Proof of Lemma \ref{lem:root-regret}}

\begin{proof}
	
WLOG, assume $C_z=C_w=1$. Let $n = \min\{k\in[K]: \ S_{k}\ge 1\}$. We have 
\begin{align*}
     \sum_{k=1}^{K}   z_k w_k 
=  \sum_{k=1}^{n} z_k w_k +   \sum_{k=n+1}^{K} z_k w_k
\le &    \sum_{k=1}^{n} w_k +    \sum_{k=n+1}^{K} z_k w_k\\
= &    S_{n} +    \sum_{k=n+1}^{K} z_k w_k\\
\le & 2  +   \sum_{k=n+1}^{K} z_k w_k.
\end{align*}
It remains to bound the second term. Using the condition that  $  z_k   S_{k-1}\le C_0 \sqrt{k}$ for all $k\in [K]$, we have 
$ z_k  \le \frac{C_0\sqrt{k} }{ S_{k-1}}$ for all $k\in [K]$ and $i\in[m]$.
Therefore,
\begin{align*}
      \sum_{k=n+1}^{K} z_k w_k 
\le &   \sum_{k=n+1}^{K} C_0\sqrt{k}  \frac{ w_k}{ S_{k-1}}\\
\le & C_0\sqrt{K}    \sum_{k=n+1}^{K}  \frac{ w_k}{ S_{k-1}}\\
\overset{(a)}{\le}  & 2C_0\sqrt{K}    \sum_{k=n+1}^{K}  \log(\frac{ S_{k}}{ S_{k-1}})\\
= &  2C_0\sqrt{K}\log(\frac{S_K}{S_n})  \le   2C_0 \sqrt{K} \log(K),
\end{align*}
where $(a)$ follows from 
$x \le 2\log (x+1)$ for $x\in[0,1]$.
\end{proof}

%\subsection{Proof of Theorem \ref{thm:main}}

% \bibliographystyle{abbrvnat}
%\bibliographystyle{plain}
% \bibliographystyle{plainnat}

%\bibliographystyle{alpha}

\end{document}